\documentclass[10pt]{article}
\usepackage[accepted]{tmlr}

\usepackage{amsmath,amsfonts,bm}

\def\eqref#1{equation~\ref{#1}}
\def\1{\bm{1}}

\DeclareMathAlphabet{\mathsfit}{\encodingdefault}{\sfdefault}{m}{sl}
\SetMathAlphabet{\mathsfit}{bold}{\encodingdefault}{\sfdefault}{bx}{n}

\usepackage{hyperref}
\hypersetup{hypertexnames=false,hidelinks}
\usepackage{url}
\usepackage{amsmath}
\usepackage{amssymb}
\usepackage{amsthm}
\usepackage{mathrsfs}
\usepackage{graphicx}
\usepackage{booktabs}
\usepackage{enumitem}
\usepackage[ruled,vlined,linesnumbered]{algorithm2e}
\usepackage[table]{xcolor}
\definecolor{discussionred}{RGB}{170,45,45}
\definecolor{revisiongreen}{RGB}{0,90,45}
\usepackage{float}
\usepackage{capt-of}
\usepackage{pgfplots}
\pgfplotsset{compat=1.18}
\usepgfplotslibrary{groupplots}
\newtheorem{theorem}{Theorem}
\newtheorem{proposition}{Proposition}
\newtheorem{lemma}{Lemma}
\newtheorem{corollary}{Corollary}
\theoremstyle{definition}

\providecommand{\PP}{\mathbb{P}}
\providecommand{\RR}{\mathbb{R}}
\providecommand{\bc}{\boldsymbol{c}}

\providecommand{\bZ}{\boldsymbol{Z}}
\providecommand{\snorm}[1]{\left\lVert #1 \right\rVert_2}
\newcommand{\WR}{W_{\text{Regret}}}
\renewcommand{\theHfigure}{main.\arabic{figure}}
\newcommand{\realexperimentdatadir}{tikz_data2}

\newcommand{\reviewexperimentdatadir}{tikz_data5}
\newcommand{\forgettingexperimentdatadir}{tikz_data6}
\newcommand{\plotrealexperimentcsv}[5]{%
\begin{tikzpicture}
\begin{axis}[
    width=\linewidth,
    height=0.58\linewidth,
    xlabel={#3},
    ylabel={#4},
    title={#5},
    grid=both,
    grid style={gray!20},
    tick label style={font=\tiny},
    label style={font=\tiny},
    title style={font=\tiny},
    legend style={font=\tiny,draw=none,fill=none},
    legend pos=north east,
]
\addplot[blue,thick] table[x=#2,y=raw,col sep=comma] {\realexperimentdatadir/#1};
\addlegendentry{Raw}
\addplot[red!75!black,thick] table[x=#2,y=raw_restart,col sep=comma] {\realexperimentdatadir/#1};
\addlegendentry{Raw + Restart}
\end{axis}
\end{tikzpicture}%
}
\pgfplotsset{
    theoryplot/.style={
        width=\linewidth,
        height=0.72\linewidth,
        grid=major,
        grid style={gray!20},
        tick label style={font=\scriptsize},
        label style={font=\small},
        title style={font=\small},
        legend style={font=\scriptsize,draw=none,fill=none},
        line width=0.9pt,
        mark size=2pt,
    },
    theory error bars/.style={
        error bars/.cd,
        y dir=both,
        y explicit,
    },
}
\newcommand{\appendixequations}{%
  \setcounter{equation}{0}%
  \renewcommand{\theequation}{\thesection.\arabic{equation}}%
  \renewcommand{\theHequation}{\thesection.\arabic{equation}}%
}
\newcommand{\appendixfigures}{%
  \setcounter{figure}{0}%
  \renewcommand{\thefigure}{\thesection.\arabic{figure}}%
  \renewcommand{\theHfigure}{appendix.\thesection.\arabic{figure}}%
}
\newcommand{\appendixtheorems}{%
  \numberwithin{theorem}{section}%
  \numberwithin{proposition}{section}%
  \numberwithin{lemma}{section}%
  \numberwithin{corollary}{section}%
  \numberwithin{definition}{section}%
  \numberwithin{remark}{section}%
  \renewcommand{\theHtheorem}{appendix.\thesection.\arabic{theorem}}%
  \renewcommand{\theHproposition}{appendix.\thesection.\arabic{proposition}}%
  \renewcommand{\theHlemma}{appendix.\thesection.\arabic{lemma}}%
  \renewcommand{\theHcorollary}{appendix.\thesection.\arabic{corollary}}%
  \renewcommand{\theHdefinition}{appendix.\thesection.\arabic{definition}}%
  \renewcommand{\theHremark}{appendix.\thesection.\arabic{remark}}%
}

\title{Online Distributional Prediction via Latent Cluster Geometry Under Drift and Corruption}

\author{\name Navyansh Mahla\thanks{Corresponding author: \texttt{navyanshmahla17@gmail.com}.} \email navyanshmahla17@gmail.com\\
\addr Indian Institute of Technology Bombay
\AND
\name Prateek Chanda \email prateekch@cse.iitb.ac.in\\
\addr Indian Institute of Technology Bombay
\AND
\name Ganesh Ramakrishnan \email ganesh@cse.iitb.ac.in\\
\addr Indian Institute of Technology Bombay}
\def\month{10} % Provisional publication date, pending AE confirmation.
\def\year{2026}
\def\openreview{\url{https://openreview.net/forum?id=DlL73vEtwX}}

\begin{document}

\maketitle

\begin{abstract}
Online learning in non-stationary streams is often formulated as tracking a
point estimate, but many applications require predicting the full
data-generating distribution. We study online distributional prediction under
drift and \textcolor{black}{bounded causal adversarial corruption}. Our approach represents each candidate law
through a latent cluster geometry: a variable-size configuration of centers
that organizes probability mass and induces a predictive distribution. A Gibbs
quasi-posterior over these configurations yields an online predictor by
posterior averaging, and the resulting variable-dimensional posterior can be
sampled with reversible-jump MCMC. \textcolor{black}{The method uses no
stream likelihood, but specifies a kernel-mixture predictive family.}
We evaluate performance by cumulative Wasserstein-1 regret against the
time-varying true law. The analysis separates two effects: corruption perturbs
the loss-based posterior update, whereas drift makes long-horizon posterior
memory stale. \textcolor{black}{We address the latter through localized memory,
and study three non-full-history variants: periodic restart, a sliding window,
and exponential discounting. We compare all three empirically against full
history.}
The bounds
decompose into a PAC-Bayesian complexity term, a corruption-sensitive posterior
perturbation term, and \textcolor{black}{a drift term driven by}
\(A_T^{\mathrm{OT}}=\sum_{t=2}^T W_2^2(p_{t-1}^*,p_t^*)\).
\textcolor{black}{Under the stated bounded-support, latent quadratic-growth,
predictive-map regularity, and localized-comparator assumptions, the restarted
predictor has sublinear high-probability cumulative Wasserstein regret if
$0\le a<1$, $q,\zeta,b<1$, $0\le\gamma<\beta<1$, and $0<h<(1-a)/2$. Here $a,q,\zeta,b$ bound the
polynomial growth of drift, predictive-family approximation error, off-order
posterior mass, and oracle-order switch cost, respectively; the learning rate
and restart length scale as $\lambda_s\asymp s^{-\beta}$ and $H\asymp T^h$.
Corruption counts obey a uniform segment-local bound $\Lambda_{r,s}=O(s^\gamma)$
through age $s$ in segment $r$, so a merely sublinear global budget is not
sufficient. The regime requires oracle-informed tuning using drift and
corruption bounds, vanishing localization radius $\varepsilon_T=o(1)$, and
fixed confidence, or $\log(1/\delta_0)=T^{o(1)}$. Sliding-window and discounted
guarantees additionally require the stated localized-memory condition.}
\end{abstract}

\section{Introduction}
\label{sec:intro}

The ubiquitous presence of high-dimensional streaming data from sensors,
networks, and online systems presents significant challenges for real-time
statistical estimation. Applications such as network monitoring, anomaly
detection, and online recommendation systems require algorithms that process
data sequentially while adapting to changing environments
\citep{aceto2013cloud, article, 10.5555/3384123, 8486321,
pmlr-v134-vanerven21a}. Real streams also exhibit features that violate many
standard assumptions: heavy tails, evolving distributions, and adversarial
corruptions \citep{6602289, 10.1145/2486001.2486031, 279982, 6235959,
10.1145/3437963.3441823, mirsky2018kitsuneensembleautoencodersonline,
10.1145/3292500.3330680, pmlr-v162-sankararaman22a}.

Most online estimation methods reduce this problem to estimating a point
quantity such as a mean, a parameter vector, or a finite-dimensional statistic.
This reduction is often too narrow. In settings such as risk control, anomaly
detection, and uncertainty quantification, the object of interest is the
predictive \emph{distribution} itself. The question is therefore not only what
the next observation is likely to be, but how probability mass is organized and
how that organization changes over time.

The central idea of this paper is to model this organization through a latent
clustering configuration space. Instead of placing a model directly on an
unrestricted sequence of distributions, we represent each possible explanation
of the current data law by a cluster-center configuration. This configuration
acts as a latent geometry for the stream: it summarizes where probability mass
is concentrated, induces a predictive law in observation space, and provides the
object over which the learner can regularize, compare, and update. The
quasi-Bayesian component is the mechanism used to maintain uncertainty over
this latent clustering space: a Gibbs quasi-posterior assigns weights to
candidate configurations according to their accumulated clustering loss.

This latent-space viewpoint also identifies the main obstruction under
non-stationarity. If the data law drifts, the relevant cluster geometry changes;
a long-horizon posterior over configurations can therefore retain stale
geometric information. Adversarial corruption perturbs the loss-based update,
while persistent drift makes a single global comparator inadequate for a moving
latent configuration. \textcolor{black}{We therefore study restart, sliding-window,
and discounted memory, using Wasserstein drift control to analyze restart.}

\noindent\textbf{Contributions.}
\begin{itemize}[nosep,leftmargin=*]
    \item We introduce a latent clustering-configuration view of online distributional prediction, where cluster-center configurations serve as structured latent geometries for the evolving data law.
    \item We derive high-probability regret bounds that explicitly expose how distributional prediction is affected by drift, corruption frequency, corruption magnitude, and data dimension through the corresponding transport, perturbation, and PAC-Bayesian complexity terms.
    \item \textcolor{black}{We study restart, sliding-window, and discounted memory within one quasi-Bayesian predictor, and identify high-probability sublinear cumulative Wasserstein-regret regimes for the restarted predictor under controlled drift and corruption.}
\end{itemize}
The quasi-Bayesian posterior update, reversible-jump sampling scheme, and
\textcolor{black}{Wasserstein drift control} are the technical mechanisms used to realize
and analyze this latent-space formulation. Algorithmically, the predictor does
not require a parametric likelihood for the streaming law, drift process, or
corruption mechanism. The theoretical guarantees are proved under the structural
latent-geometry assumptions stated in Section~\ref{sec:setup}.
\textcolor{black}{The Gibbs clustering posterior and RJMCMC machinery are inherited
from \citet{Li2016AQP}; our contributions are the induced distributional
predictor, dynamic Wasserstein analysis, corruption comparison, and memory
localization study.}

\section{Related Work}
\label{sec:related}

\noindent\textbf{Online robust estimation.}
\citet{sankararaman2023online} propose clipped SGD for online mean estimation under heavy tails, corruption, and non-stationarity, achieving $\mathcal{O}(\sqrt{T} + \Lambda_T T^{1/4})$ regret. Their bounds depend on explicit drift $\Phi_T$ and require strong convexity. We study a different distributional-prediction criterion and analyze it under latent-geometry assumptions rather than strong convexity of a mean-estimation objective. Heavy-tailed streaming estimation has been studied by \citet{tsai2022heavy}, \citet{liu2023stochastic}, and \citet{10.5555/3495724.3496985}, all within the parameter-estimation paradigm. \textcolor{black}{\citet{pmlr-v134-vanerven21a} study robust online convex optimization with outliers, measuring regret on non-outlier rounds rather than predictive Wasserstein error.}

\noindent\textbf{\textcolor{black}{Online density estimation under drift.}}
\textcolor{black}{\citet{kristan2011onlinekde} maintain an adaptive, compressed Gaussian-kernel density estimate; \citet{chen2021alokde} combine drift detection with local online KDE and introduce a noise-robust variant. \citet{gokcesu2018density} obtain log-loss regret for non-stationary exponential-family density prediction in terms of natural-parameter drift.}

\noindent\textbf{PAC-Bayesian and quasi-Bayesian methods.}
Our algorithmic foundation builds on the quasi-Bayesian online clustering of \citet{Li2016AQP}, which establishes PAC-Bayes bounds for Gibbs posteriors over cluster configurations. We extend their framework to distributional prediction with Wasserstein regret under \textcolor{black}{bounded causal corruption} and non-stationarity. \textcolor{black}{\citet{knoblauch2018doubly} use $\beta$-divergence generalized Bayes for robust online changepoint and predictive inference; their likelihood-based setting differs from our clustering-score posterior.}

\noindent\textbf{\textcolor{black}{Forgetting and adaptive regret.}}
\textcolor{black}{Fixed-share tracking \citep{herbster1998tracking} and adaptive-regret methods \citep{hazan2009changing} offer alternatives to hard restart by competing with changing experts or interval-specific comparators. Dynamic-regret analyses also address non-stationarity \citep{besbes2015non, pmlr-v134-baby21a}. Here we analyze deterministic restart and compare sliding-window and discounted memory within the same distributional predictor.}

\section{Problem Setup and Preliminaries}
\label{sec:setup}

We study online prediction of a time-varying distributional stream under
corruption. The main modeling decision is to learn in a structured
\emph{configuration space} while evaluating performance in \emph{data space}.
This section formalizes that setup and records the immediate consequences used
later in the regret analysis.

\subsection{Sequential Distributional Prediction}

At each round $t\in[T]$, an uncorrupted sample
$
x_t \sim \PP_t
$
is drawn from the current data-generating law $\PP_t$ on $\RR^d$. The learner
does not observe $x_t$ directly. Instead it receives a corrupted observation
\[
\tilde x_t = x_t + C_t,
\]
where the corruption process is \emph{causal}:
\[
C_t = \mathcal C_t\big((x_s)_{s=1}^t,(C_s)_{s=1}^{t-1}\big).
\]
Thus the adversary may depend on the entire revealed past and the current clean
point, but not on future observations. We assume bounded corruption amplitude,
\[
\|\tilde x_t-x_t\|_2 \le \delta,
\]
and denote the cumulative corruption budget by
\[
\Lambda_T := \sum_{t=1}^T \mathbf{1}\{C_t\neq 0\}.
\]

\textcolor{black}{Formally, let
\[
\mathcal F_{t-1}
:=
\sigma\!\left(x_{1:t-1},C_{1:t-1},
\text{ learner randomness through }t-1\right).
\]
We assume that \(p_t^*\) is \(\mathcal F_{t-1}\)-measurable and
\(x_t\mid\mathcal F_{t-1}\sim p_t^*\). The corruption \(C_t\) may depend on
\((\mathcal F_{t-1},x_t)\), but not on future samples, and theorem-level
restart times are deterministic.}

\textcolor{black}{We assume that clean and corrupted observations, admissible
centers, and all induced predictive laws are supported in a common ball
\(B(0,R)\), and set \(D:=2R\). Consequently,
\[
0\le\ell(\bc,x)\le D^2,
\qquad
|\ell(\bc,x)-\ell(\bc,x')|
\le 2D\|x-x'\|_2.
\]
The proofs retain the conservative valid constant \(4D\), while \(D^4\) is the
squared loss-range bound and \(D\) is the predictive-support diameter.}
This bounded-support assumption is the main regularity condition used
throughout the paper; we do not assume a prescribed drift model or separate
moment-growth conditions.

The estimator outputs a predictive law in data space. Given clean history
$x_{1:t-1}$, let $p(x_t\mid x_{1:t-1})$ denote the learner's predictive
distribution, and let $p_t^*:=\PP_t$ denote the true law at time $t$. We
measure performance by cumulative Wasserstein-1 regret:
\begin{equation}
    \WR^{\mathrm{clean}}(T)
    :=
    \sum_{t=1}^T
    W_1\!\left(p(x_t\mid x_{1:t-1}),\,p_t^*\right).
\label{eq:clean-regret}
\end{equation}
Under corrupted history $\tilde x_{1:t-1}$, the learner instead forms
$p(x_t\mid \tilde x_{1:t-1})$, and the total regret is
\begin{equation}
    \WR^{\mathrm{total}}(T)
    :=
    \sum_{t=1}^T
    W_1\!\left(p(x_t\mid \tilde x_{1:t-1}),\,p_t^*\right).
\label{eq:total-regret}
\end{equation}
This criterion is strictly stronger than parameter-level control: for any two
distributions $P,Q$ with finite first moments,
\[
\|\mathbb E_P[X]-\mathbb E_Q[X]\|_2 \le W_1(P,Q).
\]
So controlling distributional error in $W_1$ automatically controls mean error,
whereas the converse need not hold.

\subsection{Latent Clustering Representation}

The central modeling choice is to represent an evolving law through a lower
dimensional latent geometry. A cluster-center configuration is denoted by
$\bc=(c_1,\dots,c_k)$, and the clustering loss is
\begin{equation}
\ell(\bc,x):=\min_j \|c_j-x\|_2^2.
\label{eq:clustering-loss}
\end{equation}
The corresponding population risk under the true law at time $t$ is
\begin{equation}
R_t(\bc):=\mathbb E_{X\sim p_t^*}[\ell(\bc,X)].
\label{eq:population-risk}
\end{equation}
\textcolor{black}{Let $k_t^*$ denote the oracle model order and define
\[
\bc_t^* \in \arg\min_{\bc\in\mathcal C_{k_t^*}} R_t(\bc).
\]}
This oracle is an ideal benchmark: it is
defined using the true law and is therefore available to the analysis but not
to the learner.

The point of clustering here is not merely descriptive. Each configuration
serves as a latent explanation of how probability mass is organized, and each
such explanation induces a predictive law
\[
\bc \mapsto p(\cdot\mid \bc),
\]
\textcolor{black}{where
\[
p(\cdot\mid\bc)=\frac1k\sum_{j=1}^{k}K(\cdot,c_j),
\qquad K(\cdot,c)=\operatorname{Law}(Z+c).
\]
The equal weights and the common kernel parameters, including its bandwidth,
are fixed rather than inferred.}
\textcolor{black}{Translated-kernel mixtures are classical in kernel density
estimation \citep{rosenblatt1956,parzen1962}; finite Gaussian-mixture and
empirical-measure Wasserstein approximation results provide related
approximation foundations
\citep{bacharoglou2010,fournierguillin2015}. Arbitrarily fine approximation
generally requires increasing model order and decreasing bandwidth; the
residual mismatch of our fixed family is tracked by
\(A_T^{\mathrm{app}}\).}
so the method learns a distribution over latent geometries and pushes that
uncertainty forward into data space.

This viewpoint is easiest to understand by comparison with ordinary Bayesian
inference. In a fully specified Bayesian latent-variable model, one would place
a prior $\pi(d\bc)$ on configurations and update it by multiplying with a
likelihood term. Here we retain the prior over latent configurations, but do
not insist on a fully correct parametric likelihood for the streaming law.
Instead, following the quasi-Bayesian clustering framework of
\citet{Li2016AQP}, we update configurations through a loss-based Gibbs weight:
\begin{equation}
\hat\rho_{t+1}(d\bc)
\propto
\exp\!\big(-\lambda_t S_t(\bc)\big)\,\pi(d\bc),
\label{eq:quasi-posterior}
\end{equation}
where $S_t(\bc)$ is the cumulative clustering score after observing the first
$t$ samples and $\lambda_t$ is the learning-rate parameter. Thus
$\hat\rho_{t+1}$ is used to predict the next observation, not the observation
already included in $S_t$. The resulting object is a \emph{quasi-posterior}: it
retains the Bayesian architecture of prior, posterior, and predictive
distribution, but the data-fit term is generated by the clustering loss rather
than by a fully specified likelihood.

Prediction then takes place by averaging over latent configurations:
\[
\hat p_{t+1}(\cdot)
:=
\int p(\cdot\mid \bc)\,\hat\rho_{t+1}(d\bc).
\]
In particular, the prior, quasi-posterior, and KL regularization all live in
configuration space, while the regret is measured in observation space. This
separation is what makes the framework both statistically interpretable and
mathematically tractable: learning is performed in a structured latent space,
while accuracy is evaluated on the predictive law in data space.

\subsection{Configuration Space and Prior Structure}

The latent space itself is common to all algorithmic variants considered in the
paper, so we record it here once. For each admissible cluster count
$k\in\{1,\dots,p_{\max}\}$, let
\[
\mathcal C_k \subset (\RR^d)^k
\]
denote the set of $k$-center configurations allowed by the bounded-support
model. The full latent configuration space is the union
\[
\mathcal C
:=
\bigcup_{k=1}^{p_{\max}} \big(\{k\}\times \mathcal C_k\big).
\]
Thus a latent state consists of both a model order $k$ and a center geometry
$\bc\in\mathcal C_k$.
\textcolor{black}{For $\bc,\bc'\in\mathcal C_k$, configurations are identified
up to permutation and compared by
\[
d_k(\bc,\bc'):=\min_{\sigma\in\mathfrak S_k}
\left(\sum_{j=1}^k\|c_j-c'_{\sigma(j)}\|_2^2\right)^{1/2}.
\]
All local geometric assumptions use the oracle's fixed-$k$ stratum.}
\par\smallskip
\begingroup
\color{black}
For any posterior sequence $\rho=\{\rho_t\}_{t=2}^T$, define its model-order
leakage relative to the lagged oracle by
\begin{equation}
 \xi_t(\rho):=\rho_t\!\left(\mathcal C\setminus
 \big(\{k_{t-1}^*\}\times\mathcal C_{k_{t-1}^*}\big)\right),
 \qquad
 \Xi_T(\rho):=\sum_{t=2}^T\xi_t(\rho).
 \label{eq:model-order-leakage}
\end{equation}
Thus $d_k$ is never applied across different model orders: same-stratum mass
is controlled geometrically, while the remaining mass is controlled through
$\xi_t(\rho)$.
To account separately for changes in the oracle order itself, define the
one-step cost
\begin{equation}
 J_t^{\mathrm{lag}}
 :=
 \big[R_t(\bc_t^*)-R_t(\bc_{t-1}^*)\big]_+,
 \qquad
 J_T^{\mathrm{lag}}:=\sum_{t=2}^T J_t^{\mathrm{lag}},
 \label{eq:oracle-order-switch-cost}
\end{equation}
where $[x]_+:=\max\{x,0\}$. If $k_t^*=k_{t-1}^*$, optimality of
$\bc_t^*$ within that common stratum gives $J_t^{\mathrm{lag}}=0$.
The PAC--Bayes argument also uses one fixed comparator on each comparison
interval. We use the superscripts $\mathrm{nr}$ and $\mathrm{rs}$ to denote
the no-restart and restarted variants, respectively. For
$u\in\{\mathrm{nr},\mathrm{rs}\}$, let $\mathcal P_u$ be the
whole-horizon or restart-segment partition, respectively, and let $\tau_I$ be
the first round of $I\in\mathcal P_u$. Define
\begin{equation}
 J_T^{\mathrm{cmp},u}
 :=
 \sum_{I\in\mathcal P_u}\sum_{t\in I}
 \big[R_{\tau_I}(\bc_{\tau_I}^*)-R_{\tau_I}(\bc_t^*)\big]_+,
 \qquad
 J_T^u:=J_T^{\mathrm{lag}}+J_T^{\mathrm{cmp},u}.
 \label{eq:oracle-comparator-switch-cost}
\end{equation}
The first term repairs the one-step lagged-risk comparison; the second repairs
the fixed-comparator comparison within each interval. For sliding-window and
discounted scores, $J_T^{\mathrm{cmp},u}$ denotes the analogous mismatch in
the localized-memory condition stated in Section~\ref{sec:theory}. More generally, if
\[
 N_T^{\mathrm{ord}}
 :=\sum_{t=2}^T\mathbf 1\{k_t^*\ne k_{t-1}^*\},
\]
then bounded risk gives $J_T^{\mathrm{lag}}\le D^2N_T^{\mathrm{ord}}$.
For restart intervals of length at most $H$, each order change can affect at
most the remainder of one interval, so
$J_T^{\mathrm{rs}}\le D^2(H+1)N_T^{\mathrm{ord}}$.
\endgroup

\textcolor{black}{This representation is shared by all memory variants; only the
historical score changes.}

We place a prior $\pi$ on $\mathcal C$ that encodes uncertainty both about the
number of active clusters and about their locations. Following
\citet{Li2016AQP}, we write it as
\begin{equation}
    \pi(\bc)
    =
    \sum_{k=1}^{p_{\max}}
    q(k)\,\mathbf{1}_{\{\bc\in\mathcal C_k\}}\,\pi_k(\bc),
    \label{eq:configuration-prior}
\end{equation}
where $q$ is a prior on the cluster count and $\pi_k$ is a location prior on
$\mathcal C_k$. A standard choice is
\[
q(k)=\frac{\exp(-\eta k)}{\sum_{i=1}^{p_{\max}}\exp(-\eta i)},
\]
with each $\pi_k$ supported on the bounded configuration set.

All later PAC-Bayes comparisons, localized comparator constructions, and
restart arguments are carried out on this common configuration space with
respect to the prior architecture \eqref{eq:configuration-prior}.

\subsection{Drift, Lagged Risk, and Transport Complexity}

The dynamic analysis compares the learner at time $t$ not only to the current
law $p_t^*$, but also to the previous law $p_{t-1}^*$. For this reason we
introduce the lagged risk
\begin{equation}
\bar R_t(\bc):=\mathbb E_{X\sim p_{t-1}^*}[\ell(\bc,X)].
\label{eq:lagged-risk}
\end{equation}
This is simply the previous-step population risk, written inside the time-$t$
argument. Its role is to separate two effects:
\begin{enumerate}[nosep,leftmargin=*]
    \item how well the learner tracks yesterday's law, and
    \item how much the environment itself moves from $t-1$ to $t$.
\end{enumerate}
The second effect is quantified directly at the distribution level through the
\textcolor{black}{squared-Wasserstein drift measure}
\begin{equation}
A_T^{\mathrm{OT}}
:=
\sum_{t=2}^T W_2^2(p_{t-1}^*,p_t^*).
\label{eq:ot-action}
\end{equation}
This quantity plays the role of an environment complexity measure: small
$A_T^{\mathrm{OT}}$ corresponds to gradual temporal evolution, whereas large
$A_T^{\mathrm{OT}}$ corresponds to abrupt or persistent drift.
\textcolor{black}{The superscript $\mathrm{OT}$ is retained as notation; $A_T^{\mathrm{OT}}$ is a discrete squared-Wasserstein drift measure, not an optimized dynamic action.}

\subsection{Structural Assumptions}

\textcolor{black}{Beyond bounded support, the proof uses structural assumptions
on latent-geometry stability and predictive-map regularity, while explicitly
tracking predictive-family approximation.} We state them in the form actually
used later in the proofs, but it is helpful to note the more primitive
conditions from which these working inequalities arise.

\paragraph{Stable local cluster geometry.}
Around the lagged oracle, the population risk should be locally well-shaped.
At a more primitive level, one may think of this as a local identifiability
condition: the oracle configuration is an isolated local minimizer of the
lagged risk landscape. A sufficient smooth formulation is
\[
\nabla \bar R_t(\bc_{t-1}^*) = 0,
\qquad
\nabla^2 \bar R_t(\bc_{t-1}^*) \succeq \textcolor{black}{\mu_H} I
\]
for some \(\textcolor{black}{\mu_H}>0\) in a neighborhood of
\(\bc_{t-1}^*\). \textcolor{black}{A local Taylor expansion gives quadratic
growth with working constant \(\mu:=\mu_H/2\).} In the proofs we work directly
with the resulting inequality
\begin{equation}
\bar R_t(\bc)-\bar R_t(\bc_{t-1}^*)
\ge
\mu \textcolor{black}{d_k^2(\bc,\bc_{t-1}^*)}
\label{eq:quadratic-growth}
\end{equation}
for some $\mu>0$. Semantically, this means that each time slice admits a
locally stable and non-ambiguous cluster geometry: perturbing the centers away
from the oracle must increase risk in every direction. This is the mechanism
that later converts excess risk into geometric error in center space.

\paragraph{Predictive-map regularity.}
The second ingredient is a bridge from latent geometry to observation-space
prediction. We assume the predictive law varies Lipschitzly with the
configuration:
\begin{equation}
W_1\!\left(p(\cdot\mid \bc),\,p(\cdot\mid \bc')\right)
\le
L_K \textcolor{black}{d_k(\bc,\bc')},\qquad
\textcolor{black}{\bc,\bc'\in\mathcal C_k}.
\label{eq:predictive-regularity}
\end{equation}
This can itself be motivated from the translated-kernel picture below. For
example, if
\[
p(\cdot\mid \bc)=\sum_{j=1}^{p_{\max}} \omega_j K(\cdot,c_j)
\]
with common weights $(\omega_j)_j$ and $K(\cdot,c)=\mathrm{Law}(\bZ+c)$, then
coupling the same latent base variable $\bZ$ across the two configurations
gives
\[
W_1\!\left(p(\cdot\mid \bc),\,p(\cdot\mid \bc')\right)
\le
\sum_{j=1}^{p_{\max}} \omega_j \|c_j-c_j'\|_2
\le
L_K \textcolor{black}{d_k(\bc,\bc')}
\]
for a suitable norm-dependent constant $L_K$. In the analysis we keep this as
an abstract regularity assumption. Its meaning is that small perturbations of
the latent geometry produce proportionally small perturbations of the predicted
law.

\begingroup
\color{black}
Quadratic growth and predictive regularity are imposed only on
$\mathcal C_{k_{t-1}^*}$. No comparison metric or growth condition is assumed
between different strata; their contribution is retained explicitly through
$\Xi_T(\rho)$ in the Wasserstein bounds below.
\endgroup

\textcolor{black}{To track possible model mismatch, define
\begin{equation}
\alpha_t
:=
W_1\!\left(p(\cdot\mid\bc_t^*),p_t^*\right),
\qquad
A_T^{\mathrm{app}}
:=
\sum_{t=1}^T\alpha_t.
\label{eq:oracle-approximation}
\end{equation}
Exact oracle realizability is the special case $A_T^{\mathrm{app}}=0$.}
The learner does not observe \(\bc_t^*\).

\textcolor{black}{Equation~\eqref{eq:predictive-regularity} controls estimation
within the predictive family, while \eqref{eq:oracle-approximation} records its
distance from the true law. Lemma~\ref{lem:appendix-oracle-approximation}
decomposes this gap into quantization, unequal-mass, and kernel-smoothing
terms. In particular,
\[
A_T^{\mathrm{app}}
\le
\sum_{t=1}^T
\left[
\sqrt{R_t(\bc_t^*)}
+
\frac D2\sum_{j=1}^{k_t^*}
\left|
p_t^*(V_j(\bc_t^*))-\frac1{k_t^*}
\right|
+
\mathbb E\|Z\|_2
\right].
\]
\textcolor{black}{Individually, bounded support gives
\[
0\le R_t(\bc_t^*)\le D^2,
\qquad
0\le\sqrt{R_t(\bc_t^*)}\le D,
\]
while the fact that both weight vectors lie on the probability simplex gives
\[
0\le
\frac D2\sum_{j=1}^{k_t^*}
\left|
p_t^*(V_j(\bc_t^*))-\frac1{k_t^*}
\right|
\le
D\left(1-\frac1{k_t^*}\right)
<D.
\]}
The later ansatz $A_T^{\mathrm{app}}\asymp T^q$ is an explicit condition on
the growth of this cumulative gap, not a consequence of the general bound.}

\paragraph{Common translated-kernel representation.}
\textcolor{black}{\textcolor{black}{For a geometric interpretation of this drift measure, a useful sufficient}
setting is that the true law can be written as a mixture of translated copies
of a common kernel:}
\begin{equation}
p_t^*
=
\sum_{j=1}^{p_{\max}} w_{t,j} K(\cdot,c_{t,j}^*),
\qquad
K(\cdot,c)=\mathrm{Law}(\bZ+c).
\label{eq:translated-kernel-mixture}
\end{equation}
This assumption is not a parametric drift model. Rather, it gives a geometric
representation in which temporal evolution can be understood through center
motion, weight reallocation, and birth or death of active components.
Its main technical role is that it makes the \textcolor{black}{drift measure} explicitly comparable to
center motion. Indeed, if $\Pi_t\in\Gamma(w_{t-1},w_t)$ couples the mixture
weights at times $t-1$ and $t$, then transporting the common base variable
$\bZ$ along the coupled components yields
\[
W_2^2(p_{t-1}^*,p_t^*)
\le
\sum_{i,j}\Pi_{t,ij}\|c_{t-1,i}^*-c_{t,j}^*\|_2^2,
\]
and therefore
\begin{equation}
W_2^2(p_{t-1}^*,p_t^*)
\le
\min_{\Pi_t\in\Gamma(w_{t-1},w_t)}
\sum_{i,j}\Pi_{t,ij}\|c_{t-1,i}^*-c_{t,j}^*\|_2^2.
\label{eq:ot-coupling-min}
\end{equation}
So the translated-kernel representation is what lets the \textcolor{black}{Wasserstein drift measure}
inherit a direct geometric interpretation in terms of evolving cluster
configurations.

\subsection{Immediate Consequences Used Later}

The preceding setup immediately yields the facts repeatedly used in the proofs.

\paragraph{Loss stability in the sample variable.}
Under the bounded-support convention, for each fixed configuration $\bc$ the
map $x\mapsto \ell(\bc,x)$ is $4D$-Lipschitz:
\[
|\ell(\bc,x)-\ell(\bc,x')|
\le
4D\,\|x-x'\|_2.
\]
Consequently, Kantorovich--Rubinstein duality implies
\begin{equation}
\left|
\mathbb E_{p_{t-1}^*}[\ell(\bc,X)]
-
\mathbb E_{p_t^*}[\ell(\bc,X)]
\right|
\le
4D\,W_1(p_{t-1}^*,p_t^*).
\label{eq:risk-shift-bound}
\end{equation}
This is the basic way in which distribution drift enters the risk analysis:
changes in the law induce controlled changes in the risk landscape.

\paragraph{Geometric control from excess risk.}
The quadratic-growth assumption yields
\begin{equation}
\textcolor{black}{d_k(\bc,\bc_{t-1}^*)}
\le
\frac{1}{\sqrt{\mu}}
\sqrt{\bar R_t(\bc)-\bar R_t(\bc_{t-1}^*)}.
\label{eq:excess-risk-geometric-control}
\end{equation}
Thus small lagged excess risk implies small geometric error in configuration
space. Combined with predictive-map regularity and
\textcolor{black}{the explicit approximation budget}, this gives a route from
population risk control to Wasserstein prediction error.

\paragraph{Pathwise control of drift.}
Since $W_1\le W_2$, the \textcolor{black}{squared-Wasserstein drift measure} controls the total
distributional motion:
\[
\sum_{t=2}^T W_1(p_{t-1}^*,p_t^*)
\le
\sum_{t=2}^T W_2(p_{t-1}^*,p_t^*)
=
\sum_{t=2}^T \sqrt{W_2^2(p_{t-1}^*,p_t^*)}.
\]
Applying Cauchy--Schwarz to the sequence
$\big(\sqrt{W_2^2(p_{t-1}^*,p_t^*)}\big)_{t=2}^T$ gives
\begin{equation}
\sum_{t=2}^T \sqrt{W_2^2(p_{t-1}^*,p_t^*)}
\le
\sqrt{(T-1)\sum_{t=2}^T W_2^2(p_{t-1}^*,p_t^*)}
=
\sqrt{(T-1)A_T^{\mathrm{OT}}}.
\label{eq:pathwise-drift-bound}
\end{equation}
\textcolor{black}{Thus
$\sum_{t=2}^T W_1(p_{t-1}^*,p_t^*)\le\sqrt{(T-1)A_T^{\mathrm{OT}}}$.
This path-length control is the role of the drift measure in the proof;
no dynamic-OT argument is used.}

\paragraph{Why restart will appear later.}
The fixed-comparator quasi-Bayesian theorem is naturally aligned with a static
benchmark, while our target oracle $\bc_t^*$ moves with time. Restart-based
segmentation resolves this mismatch by replacing one whole-horizon comparator
with a sequence of local fixed comparators. The later restart analysis will
therefore use segment-start oracles and localized comparator neighborhoods in
configuration space rather than a single global point-mass benchmark.

\section{Quasi-Bayesian Clustering Algorithms}
\label{sec:algorithms}

\subsection{Common Quasi-Bayesian Framework}

We now turn from the common modeling setup to the learning rules built on top of
it. \textcolor{black}{All memory variants} share the same latent space, prior
architecture, and predictive target. In each case, the learner maintains a
probability distribution on the common configuration space $\mathcal C$ and uses
it to produce a predictive law in data space.

The shared quasi-Bayesian template is the following. Let $S_t(\bc)$ denote a
cumulative score assigned to configuration $\bc$ after observing the stream up
to time $t$. The learner forms a Gibbs-type update for the next prediction,
\begin{equation}
    \hat\rho_{t+1}(d\bc)
    =
    \frac{\exp\!\big(-\lambda_t S_t(\bc)\big)\,\pi(d\bc)}
    {\int_{\mathcal C}\exp\!\big(-\lambda_t S_t(\bc')\big)\,\pi(d\bc')},
    \label{eq:generic-gibbs-update}
\end{equation}
where $\pi$ is the prior from \eqref{eq:configuration-prior} and $\lambda_t>0$
is the learning-rate or inverse-temperature parameter. The associated
predictive distribution is then
\begin{equation}
    \hat p_{t+1}(\cdot)
    =
    \int_{\mathcal C} p(\cdot\mid \bc)\,\hat\rho_{t+1}(d\bc).
    \label{eq:generic-predictive-law}
\end{equation}

\begingroup
\color{black}
The variants differ only in memory:
\begin{enumerate}[nosep,leftmargin=*]
    \item \emph{full history} retains all losses;
    \item \emph{restart} periodically resets the score;
    \item \emph{sliding window} retains the latest $H$ losses;
    \item \emph{discounting} geometrically downweights older losses.
\end{enumerate}
All share the same latent model, prior, loss, and predictive map.
\endgroup

\subsection{Vanilla Update}

The baseline learner instantiates the common framework with a first-order
Gibbs score over cluster configurations. Let $z_t$ denote the observation
available at time $t$. In the algorithmic description below, $z_t$ may be read
as the incoming stream seen by the learner. In the clean analysis it will later
be instantiated as $x_t$, while in the corruption-aware version it will be
instantiated as $\tilde x_t$.

Initialize $S_0(\bc)=0$. For $t\ge 1$, define the cumulative score recursively
by
\begin{equation}
    S_t(\bc)
    :=
    S_{t-1}(\bc)
    +
    \ell(\bc,z_t).
    \label{eq:vanilla-cumulative-score}
\end{equation}
The corresponding Gibbs quasi-posterior at round $t+1$ is
\begin{equation}
    d\hat\rho_{t+1}(\bc)
    =
    \frac{\exp\!\big(-\lambda_t S_t(\bc)\big)}
    {\int_{\mathcal C}\exp\!\big(-\lambda_t S_t(\bc')\big)\,d\pi(\bc')}\,d\pi(\bc),
    \label{eq:vanilla-quasi-posterior}
\end{equation}
or equivalently $d\hat\rho_{t+1}(\bc)\propto \exp(-\lambda_t S_t(\bc))\,d\pi(\bc)$.
The predictive distribution used before observing the next sample is
\begin{equation}
    p(z_{t+1}\mid z_{1:t})
    =
    \int_{\mathcal C} p(z_{t+1}\mid \bc)\,\hat\rho_{t+1}(\bc)\,d\bc.
    \label{eq:vanilla-predictive-law}
\end{equation}

This is a first-order version of the quasi-Bayesian clustering update. The
second-order variance correction appearing in the original online clustering
construction of \citet{Li2016AQP} is not included in the score. In the present
paper, the score is kept aligned with the cumulative data-fit loss, while
temporal instability is handled through restart and the dynamic transport term.

\begin{algorithm}[H]
    \caption{Vanilla Quasi-Bayesian Clustering Update}
    \label{alg:vanilla-qb}
    \KwIn{$p_{\max}$, prior $\pi\in\mathscr P(\mathcal C)$, learning-rate schedule $\{\lambda_t\}_{t=0}^{T-1}$, initial score $S_0=0$}
    \KwOut{Predictive distributions $\{p(z_{t+1}\mid z_{1:t})\}_{t=0}^{T-1}$}
    \For{$t\leftarrow 0$ \KwTo $T-1$}{
        Form the quasi-posterior $d\hat\rho_{t+1}(\bc)\propto \exp(-\lambda_t S_t(\bc))\,d\pi(\bc)$\;
        Draw cluster configurations from $\hat\rho_{t+1}$ via RJMCMC\;
        Output the predictive law $p(z_{t+1}\mid z_{1:t}) = \int p(z_{t+1}\mid \bc)\,\hat\rho_{t+1}(\bc)\,d\bc$\;
        Observe the next sample $z_{t+1}$\;
        Update the cumulative score using \eqref{eq:vanilla-cumulative-score}\;
    }
\end{algorithm}

Because the latent space $\mathcal C$ includes configurations with different
numbers of clusters, the target law in \eqref{eq:vanilla-quasi-posterior} is a
variable-dimension distribution. Following \citet{Li2016AQP}, we approximate
it by reversible-jump MCMC. At each round, the chain targets
$\hat\rho_{t+1}$ and uses three types of moves: birth moves that increase the
cluster count, death moves that decrease it, and within-model updates that
modify the center locations while keeping $k$ fixed. The normalizing constant
in \eqref{eq:vanilla-quasi-posterior} need not be computed explicitly, since it
cancels in the Metropolis--Hastings acceptance ratio.

If $\bc^{(1)},\dots,\bc^{(M)}$ are RJMCMC draws from $\hat\rho_{t+1}$ after
burn-in, then the predictive integral in \eqref{eq:vanilla-predictive-law} is
approximated by the empirical average
\begin{equation}
    p(z_{t+1}\mid z_{1:t})
    \approx
    \frac{1}{M}\sum_{m=1}^M p(z_{t+1}\mid \bc^{(m)}).
    \label{eq:mc-predictive-approx}
\end{equation}
In other words, the algorithm first samples plausible latent cluster
configurations from the quasi-posterior and then averages the corresponding
configuration-wise predictive laws. This is the practical mechanism through
which uncertainty in configuration space is pushed forward into a predictive
distribution in data space.

\subsection{Restarted Update}

The second modification concerns temporal memory rather than the loss itself.
Fix restart times
\[
1=\tau_0<\tau_1<\cdots<\tau_m=T+1
\]
and let
\[
I_r:=\{\tau_{r-1},\dots,\tau_r-1\}
\]
denote segment $r$. On each segment, the quasi-posterior is reinitialized and
updated only with the observations observed since the most recent restart.

For the restarted version, the same causal convention is applied locally within
each segment. For $t\in I_r$, define the within-segment age
\[
    a_r(t):=t-\tau_{r-1},
\]
and define the segmentwise past score, with the empty sum equal to zero, by
\begin{equation}
    S_{t-1}^{r}(\bc)
    :=
    \sum_{s=\tau_{r-1}}^{t-1}
    \ell(\bc,z_s),
    \qquad t\in I_r.
    \label{eq:restart-score}
\end{equation}
Before observing $z_t$, the restarted quasi-posterior on segment $r$ is then
\begin{equation}
    d\hat\rho_t^{r}(\bc)
    =
    \frac{\exp\!\big(-\lambda_{a_r(t)} S_{t-1}^{r}(\bc)\big)}
    {\int_{\mathcal C}\exp\!\big(-\lambda_{a_r(t)} S_{t-1}^{r}(\bc')\big)\,d\pi(\bc')}\,d\pi(\bc),
    \qquad t\in I_r.
    \label{eq:restart-posterior}
\end{equation}
After $z_t$ is observed, the score is updated to
$S_t^r(\bc)=S_{t-1}^r(\bc)+\ell(\bc,z_t)$ and used only for later predictions
in the same segment.

The predictive law is obtained by integrating the same map $p(\cdot\mid \bc)$
against the restarted quasi-posterior on the current segment:
\[
    p(z_t\mid z_{\tau_{r-1}:t-1})
    =
    \int_{\mathcal C}p(z_t\mid \bc)\,\hat\rho_t^r(d\bc),
    \qquad t\in I_r.
\]
The role of restart is therefore purely temporal: it limits how much past data
can influence the current update, without changing the latent representation,
loss, prior, or configuration-wise predictive map.

\begingroup
\color{black}
\subsection{Sliding-Window and Discounted Updates}

For window length $H$, the causal sliding-window score is
\[
S_{t-1}^{\mathrm{sw}}(\bc)
=\sum_{s=\max\{1,t-H\}}^{t-1}\ell(\bc,z_s).
\]
For discount factor $\rho\in(0,1)$, define
\[
S_{t-1}^{\mathrm{dc}}(\bc)
=\sum_{s=1}^{t-1}\rho^{t-1-s}\ell(\bc,z_s),
\qquad
S_t^{\mathrm{dc}}=\rho S_{t-1}^{\mathrm{dc}}+\ell(\bc,z_t).
\]
Substituting either score into \eqref{eq:generic-gibbs-update} gives the
corresponding predictor. We use $\rho=\exp(-1/H)$, making $H$ the effective
discounted-memory scale.
\endgroup

\section{Theoretical Guarantees}
\label{sec:theory}

\subsection{Total Regret}

We now state the total-regret guarantees for the base and restarted predictors
introduced in Section~\ref{sec:algorithms}. To keep the theorem statements
readable, we first isolate the corruption-side and clean-side terms that recur
in both cases.
The terminology follows the decomposition in Appendix~\ref{app:proofs}. At
time \(t\), the learner using the corrupted stream predicts from
\(p(x_t\mid \tilde x_{1:t-1})\), while the ideal clean-history learner would
predict from \(p(x_t\mid x_{1:t-1})\). Thus the total Wasserstein regret is
controlled by two effects:
\[
    W_1\!\left(p(x_t\mid \tilde x_{1:t-1}),p_t^*\right)
    \le
    W_1\!\left(p(x_t\mid \tilde x_{1:t-1}),p(x_t\mid x_{1:t-1})\right)
    +
    W_1\!\left(p(x_t\mid x_{1:t-1}),p_t^*\right).
\]
The first term is the corruption-side contribution: it measures how much the
posterior predictive law changes because the history was corrupted. The second
term is the clean-side contribution: it measures how well the clean
quasi-Bayesian predictor tracks the moving true law \(p_t^*\) under drift.
The clean dynamic analysis starts from \(t=2\), so we retain the initial
prediction error
\(\Delta_1:=W_1(p(x_1\mid\emptyset),p_1^*)\le D\) explicitly below.

The bounds below are high-probability statements over the clean sample path
and the idealized learner randomness. Corruption is treated as fixed, possibly
adversarial, subject to the stated amplitude and budget constraints. The
confidence level is denoted by \(\delta_0\in(0,1)\).

\textcolor{black}{Let $L_\ell$ be the uniform constant satisfying
$|\ell(\bc,x)-\ell(\bc,x')|\le L_\ell\|x-x'\|_2$.}
For the no-restart update, define the global corruption contribution
\begin{equation}
    \mathcal E_T^{\mathrm{nr}}
    :=
    \frac{D}{2}\sum_{t=1}^T
    \left[
    \exp\!\left(
    2\lambda_{t-1}L_\ell\delta\,\Lambda_{t-1}
    \right)-1
    \right],
    \label{eq:no-restart-corruption-term}
\end{equation}
where \(\Lambda_{t-1}\) is the number of corrupted observations in the global
history up to time \(t-1\). For the restarted update, let
\[
    \Lambda_{r,t-1}
    :=
    \sum_{s=\tau_{r-1}}^{t-1}\mathbf 1\{x_s\ne \tilde x_s\},
    \qquad t\in I_r,
\]
be the number of corruptions seen inside the current restart segment before
time \(t\). The segment-local corruption contribution is
\begin{equation}
    \mathcal E_T^{\mathrm{rs}}
    :=
    \frac{D}{2}
    \sum_{r=1}^{m}
    \sum_{t\in I_r}
    \left[
    \exp\!\left(
    2\lambda_{a_r(t)}L_\ell\delta\,\Lambda_{r,t-1}
    \right)-1
    \right].
    \label{eq:restart-corruption-term}
\end{equation}
The no-restart and restarted corruption comparisons that produce
\(\mathcal E_T^{\mathrm{nr}}\) and \(\mathcal E_T^{\mathrm{rs}}\) are proved in
Propositions~\ref{prop:appendix-raw-corruption}
and~\ref{prop:appendix-restart-corruption}, respectively.

For the clean dynamic terms, let
$\{\lambda_t\}_{t=0}^{T-1}$ be a positive non-increasing learning-rate
schedule, and define
\[
    \bar\lambda_T := \sum_{t=0}^{T-1}\lambda_t.
\]
For the restarted scheme with equal segment length $H$, we reuse the same
segment-local schedule $\{\lambda_s\}_{s=0}^{H-1}$ on each segment, and write
\[
    \bar\lambda_H := \sum_{s=0}^{H-1}\lambda_s.
\]
For notational simplicity, the displayed restarted bounds write \(T/H\) for
the number of restart segments; if \(H\) does not divide \(T\), the same proof
uses \(m=\lceil T/H\rceil\) in place of \(T/H\).
The terms proportional to $\bar\lambda_T$ and $\bar\lambda_H$ are proof-side
bounded-loss concentration terms. They are not part of the posterior score
defined in Section~\ref{sec:algorithms}.

For a posterior sequence \(\rho=\{\rho_t\}_{t=2}^T\), define its current-risk
excess by
\[
    E_T^{\mathrm{cur}}(\rho)
    :=
    \sum_{t=2}^T
    \mathbb E_{\bc\sim\rho_t}
    \big[R_t(\bc)-R_t(\bc_t^*)\big].
\]
\textcolor{black}{Let $L_c$ denote a uniform within-stratum
risk-Lipschitz constant:
\[
 |R_t(\bc)-R_t(\bc')|\le L_c\,d_k(\bc,\bc'),
 \qquad \bc,\bc'\in\mathcal C_k.
\]}
\textcolor{black}{For each comparison interval (I), we assume a fixed
comparator \(\bc^\dagger\) and a localized distribution \(\nu_I\ll\pi\) such
that
\[
\mathrm{KL}(\nu_I\,\|\,\pi)\le C,
\qquad
\mathbb E_{\bc\sim\nu_I}R_t(\bc)
\le R_t(\bc^\dagger)+L_c\varepsilon,\quad t\in I.
\]
The whole-horizon analysis uses one such comparator globally, while the
restarted analysis uses one per restart segment.}
The following lemma is the latent-risk input used by the total-regret
theorems. Its proof is deferred to Appendices~\ref{app:no-restart-clean-proof}
and~\ref{app:restart-clean-proof}, using the high-probability PAC-Bayes
bound from Appendix~\ref{app:clean-pac-bayes}. More specifically, the
no-restart and restarted current-risk excess terms are controlled in
Lemmas~\ref{lem:appendix-no-restart-current-risk}
and~\ref{lem:appendix-restart-current-risk}.

\begin{lemma}[High-probability latent excess-risk brackets]
    \label{lem:main-latent-excess-brackets}
    With probability at least \(1-\delta_0\), the no-restart posterior sequence
    \(\hat\rho^{\mathrm{nr}}\) satisfies
    \[
        E_T^{\mathrm{cur}}(\hat\rho^{\mathrm{nr}})
        +
        8D\sqrt{(T-1)A_T^{\mathrm{OT}}}
        \le
        \mathcal B_{T,\delta_0}^{\mathrm{nr}}(\{\lambda_t\})
        +\textcolor{black}{J_T^{\mathrm{cmp},\mathrm{nr}}},
    \]
    where
    \begin{equation}
        \begin{aligned}
        \mathcal B_{T,\delta_0}^{\mathrm{nr}}(\{\lambda_t\})
        :=
        &L_cT\varepsilon
        +
        \frac{2C+\log(2/\delta_0)}{\lambda_{T-1}}
        +
        \frac{D^4}{2}\bar\lambda_T
        \\
        &\qquad
        +
        \frac{D^4T\lambda_{T-1}}{8}
        +
        D^2\sqrt{\frac{T}{2}\log\frac{2}{\delta_0}}
        +
        8DT\sqrt{(T-1)A_T^{\mathrm{OT}}}
        \\
        &\qquad
        +
        8D\sqrt{(T-1)A_T^{\mathrm{OT}}}.
        \end{aligned}
        \label{eq:main-no-restart-latent-bracket}
    \end{equation}
    For deterministic equal-length restart windows, the restarted posterior
    sequence \(\hat\rho^{\mathrm{rs}}\) satisfies
    \[
        E_T^{\mathrm{cur}}(\hat\rho^{\mathrm{rs}})
        +
        8D\sqrt{TA_T^{\mathrm{OT}}}
        \le
        \mathcal B_{T,\delta_0}^{\mathrm{rs}}(\{\lambda_s\},H)
        +\textcolor{black}{J_T^{\mathrm{cmp},\mathrm{rs}}},
    \]
    where
    \begin{equation}
        \begin{aligned}
        \mathcal B_{T,\delta_0}^{\mathrm{rs}}(\{\lambda_s\},H)
        :=
        &L_cT\varepsilon
        +
        \frac{2(T/H)C+\log(2/\delta_0)}{\lambda_{H-1}}
        +
        \frac{D^4}{2}\frac{T}{H}\bar\lambda_H
        \\
        &\qquad
        +
        \frac{D^4T\lambda_{H-1}}{8}
        +
        D^2\sqrt{\frac{T}{2}\log\frac{2}{\delta_0}}
        +
        8DH\sqrt{TA_T^{\mathrm{OT}}}
        \\
        &\qquad
        +
        8D\sqrt{TA_T^{\mathrm{OT}}}.
        \end{aligned}
        \label{eq:main-restart-latent-bracket}
    \end{equation}
\end{lemma}

The complexity \(C\) in Lemma~\ref{lem:main-latent-excess-brackets} is the KL
cost of a localized comparator. If \(\bc^\star\in\mathcal C_k\), define the
localized comparator by truncating the prior \eqref{eq:configuration-prior} to a
radius-\(r\) ball around \(\bc^\star\):
\[
    \nu_{\bc^\star,r}(d\bc)
    =
    \frac{
    \mathbf 1\{\bc\in\mathcal C_k,\ \textcolor{black}{d_k(\bc,\bc^\star)}\le r\}\,\pi(d\bc)
    }{
    \pi\{\bc\in\mathcal C_k:\textcolor{black}{d_k(\bc,\bc^\star)}\le r\}
    }.
\]
Then
\[
    \mathrm{KL}(\nu_{\bc^\star,r}\|\pi)
    =
    \log
    \frac{1}{
    \pi\{\bc\in\mathcal C_k:\textcolor{black}{d_k(\bc,\bc^\star)}\le r\}
    }
    \le C.
\]
The prior-mass calculation in Lemma~\ref{lem:appendix-localized-comparator}
gives, up to constants, where \(R\) is the bounded-support radius introduced in
Section~\ref{sec:setup},
\[
    C=C_{k,d}(r)
    \lesssim
    \log\frac{1}{q(k)}
    +
    kd\log\frac{2R}{r}.
\]
Thus dimension enters the regret bounds through this PAC-Bayesian complexity
factor.

The Wasserstein conversion in
Corollary~\ref{cor:appendix-learning-to-yesterday} turns these latent brackets
into clean dynamic regret terms. Define
\begin{equation}
    \mathcal G_{T,\delta_0}^{\mathrm{nr}}(\{\lambda_t\})
    :=
    \frac{L_K}{\sqrt{\mu}}
    \sqrt{
    (T-1)\left[
    \mathcal B_{T,\delta_0}^{\mathrm{nr}}(\{\lambda_t\})
    +\textcolor{black}{J_T^{\mathrm{nr}}}
    +\textcolor{black}{D^2\Xi_T(\hat\rho^{\mathrm{nr}})}
    \right]
    }
    +\textcolor{black}{D\Xi_T(\hat\rho^{\mathrm{nr}})}
    +
    \sqrt{(T-1)A_T^{\mathrm{OT}}}
    +
    \textcolor{black}{A_T^{\mathrm{app}}},
    \label{eq:no-restart-clean-term}
\end{equation}
and
\begin{equation}
    \mathcal G_{T,\delta_0}^{\mathrm{rs}}(\{\lambda_s\},H)
    :=
    \frac{L_K}{\sqrt{\mu}}
    \sqrt{
    T\left[
    \mathcal B_{T,\delta_0}^{\mathrm{rs}}(\{\lambda_s\},H)
    +\textcolor{black}{J_T^{\mathrm{rs}}}
    +\textcolor{black}{D^2\Xi_T(\hat\rho^{\mathrm{rs}})}
    \right]
    }
    +\textcolor{black}{D\Xi_T(\hat\rho^{\mathrm{rs}})}
    +
    \sqrt{TA_T^{\mathrm{OT}}}
    +
    \textcolor{black}{A_T^{\mathrm{app}}}.
    \label{eq:restart-clean-term}
\end{equation}

\begin{theorem}[Quasi-Bayesian predictor]
    \label{thm:raw-total}
    For the base update without restart, with probability at least
    \(1-\delta_0\),
    \begin{equation}
        \WR^{\mathrm{total}}(T)
        \le
        \Delta_1
        +
        \mathcal E_T^{\mathrm{nr}}
        +
        \mathcal G_{T,\delta_0}^{\mathrm{nr}}(\{\lambda_t\}).
        \label{eq:raw-total-bound}
    \end{equation}
\end{theorem}

\begin{theorem}[Restarted quasi-Bayesian predictor]
    \label{thm:raw-restart-total}
    For the restarted update with deterministic equal-length restart windows,
    with probability at least \(1-\delta_0\),
    \begin{equation}
        \WR^{\mathrm{total}}(T)
        \le
        \Delta_1
        +
        \mathcal E_T^{\mathrm{rs}}
        +
        \mathcal G_{T,\delta_0}^{\mathrm{rs}}(\{\lambda_s\},H).
        \label{eq:raw-restart-total-bound}
    \end{equation}
\end{theorem}

The two theorems expose the role of temporal localization. The corruption term
\eqref{eq:no-restart-corruption-term} uses the full corrupted history, whereas
\eqref{eq:restart-corruption-term} uses only corruptions accumulated inside the
current restart segment. Restart therefore localizes both the clean drift
comparison and the corruption perturbation. On the clean side, it replaces the
whole-horizon memory penalty \eqref{eq:no-restart-clean-term} by the localized
dynamic bound \eqref{eq:restart-clean-term}.
\textcolor{black}{Sliding-window and discounted scores also limit stale influence.
Their sublinear regimes are stated below under the corresponding localized
PAC--Bayes control.}
\par\smallskip
\begingroup
\color{black}
The finite-time bounds retain both the realized leakage $\Xi_T$ and the
oracle-order switch cost $J_T^u$; they do not assume that the quasi-posterior
identifies the oracle order or that the oracle order is fixed. A sufficient
evidence condition for decreasing leakage is given in
Proposition~\ref{prop:appendix-order-evidence}; the sublinear regime below
states separately the required growth rate of $\Xi_T$.
\endgroup
\textcolor{black}{Our robustness notion is perturbation stability: the algorithm
does not detect corruptions, while restart limits their persistence in the
cumulative score.}
All appearances of $A_T^{\mathrm{OT}}$ should be read as drift penalties: they
measure the amount of transport energy required to move the true distribution
path over time. When this \textcolor{black}{drift measure} is small, the stream evolves gradually and the
restart term can remain favorable. When it is large, the true law is spending
substantial transport energy, and the regret bound correspondingly reflects a
harder tracking problem.

\subsection{Sublinear Regimes and Configuration Comparison}

To compare the base and restarted configurations at the level of long-horizon growth, we use
the polynomial ansatz
\begin{equation}
    \begin{aligned}
    \lambda_t &\asymp t^{-\beta}, &
    \Lambda_T &\asymp T^\gamma, &
    \Lambda_{r,a}&\lesssim a^\gamma, &
    A_T^{\mathrm{OT}} &\asymp T^a,\\
    \textcolor{black}{A_T^{\mathrm{app}}}&\textcolor{black}{\asymp T^q}, &
    \textcolor{black}{\Xi_T(\hat\rho^u)}&\textcolor{black}{\lesssim T^\zeta}, &
    \textcolor{black}{J_T^u}&\textcolor{black}{\lesssim T^b}, &
    H&\asymp T^h.
    \end{aligned}
    \label{eq:polynomial-ansatz}
\end{equation}
Here $\beta$ controls the learning-rate decay, $\gamma$ the growth of the
cumulative corruption budget, $a$ the growth of the \textcolor{black}{squared-Wasserstein
drift measure}, \textcolor{black}{$q$ the growth of the cumulative approximation error},
\textcolor{black}{$\zeta$ the growth of cumulative posterior mass on
non-oracle model orders and $b$ the growth of the oracle-order switch cost},
and $h$ the growth of the restart window. The quantity
$J_T^u$ combines the one-step and fixed-comparator costs in
\eqref{eq:oracle-order-switch-cost}--\eqref{eq:oracle-comparator-switch-cost}.
For restart, $N_T^{\mathrm{ord}}=O(T^s)$ implies
$J_T^{\mathrm{rs}}=O(T^{h+s})$. Here
$u\in\{\mathrm{nr},\mathrm{rs},\mathrm{sw},\mathrm{dc}\}$ denotes the memory
rule under consideration. For the restarted
configuration, the corruption condition is interpreted segment-locally:
\(\Lambda_{r,a}\) counts corruptions inside segment \(r\) up to within-segment
age \(a\). This is the corruption quantity that appears in
\eqref{eq:restart-corruption-term}.

\begingroup
\color{black}
For $u\in\{\mathrm{sw},\mathrm{dc}\}$, let $H$ denote the window length or
effective discounted-memory scale. Define the active corruption mass by
\[
 \Lambda_t^{\mathrm{sw}}
 =\sum_{s=(t-H)\vee1}^{t-1}\mathbf 1\{x_s\ne\tilde x_s\},
 \qquad
 \Lambda_t^{\mathrm{dc}}
 =\sum_{s=1}^{t-1}\rho^{t-1-s}\mathbf 1\{x_s\ne\tilde x_s\}.
\]
We use the following localized-memory condition: uniformly in $t$,
$\Lambda_t^u\lesssim H^\gamma$, and the clean latent bracket satisfies
\begin{equation}
 \mathcal B_{T,\delta_0}^{u}
 =\widetilde O\!\left(
 TH^{\beta-1}+TH^{-\beta}+H^\beta+\sqrt T
 +H\sqrt{TA_T^{\mathrm{OT}}}+\sqrt{TA_T^{\mathrm{OT}}}
 \right)+L_cT\varepsilon_T .
 \label{eq:localized-memory-control}
\end{equation}
This is the sliding-window or discounted analogue of the restarted bracket
\eqref{eq:main-restart-latent-bracket}; it makes explicit the additional
condition under which the extension below holds. When the oracle order changes,
the corresponding fixed-comparator mismatch is retained separately in
$J_T^{\mathrm{cmp},u}$, exactly as for restart.
\endgroup

\begin{proposition}[Sublinear regime comparison]
    \label{prop:sublinear-regime-comparison}
    Under the bounds of Theorems~\ref{thm:raw-total} and
    \ref{thm:raw-restart-total}
    and the scaling ansatz \eqref{eq:polynomial-ansatz}, suppose the
    localized-comparator approximation radius is chosen so that
    $\varepsilon=\varepsilon_T=o(1)$ and the confidence level is fixed, or more
    generally satisfies $\log(1/\delta_0)=T^{o(1)}$. Then the following hold.
    \begin{enumerate}
        \item The no-restart theorem, with the whole-horizon comparator used
        here, does not certify a sublinear high-probability total-regret bound under
        the current proof technique.
        \item The restarted configuration admits a sublinear high-probability
        regret
        regime whenever one can choose exponents such that
        \begin{equation}
            \textcolor{black}{0\le a<1},
            \qquad
            \textcolor{black}{q<1},
            \qquad
            \textcolor{black}{\zeta<1},
            \qquad
            \textcolor{black}{b<1},
            \qquad
            \textcolor{black}{0\le\gamma<\beta<1,\qquad 0<h<\frac{1-a}{2}}.
            \label{eq:restart-sublinear-regime}
        \end{equation}
        \item \textcolor{black}{Under the localized-memory condition
        \eqref{eq:localized-memory-control}, the sliding-window and discounted
        configurations admit the same sufficient sublinear regime
        \eqref{eq:restart-sublinear-regime}.}
    \end{enumerate}
\end{proposition}

\begin{proof}[Discussion of Proposition~\ref{prop:sublinear-regime-comparison}]
    \textcolor{black}{The learning-rate and memory exponents need not
    satisfy $\beta<h$: corruption requires $\beta>\gamma$, clean terms require
    $0<\beta<1$ and $0<\beta h<1$, and drift requires $h<(1-a)/2$.
    The stated regime implies all these conditions.}
    The proof is given in Appendix~\ref{app:sublinear-proof}. The main point is
    that the localized corruption exponent scales as
    $H^{\gamma-\beta}$, while the localized drift term is sublinear when the
    memory scale grows more slowly than $T^{(1-a)/2}$.
\end{proof}

\textcolor{black}{This proposition is a conditional characterization of the
sublinear regime: \(a\) and \(\gamma\) describe the environment for analysis
and are neither observed nor used by the algorithm. Selecting \(\beta\) and
\(h\) from known values or bounds for these exponents would be
\emph{oracle-informed}; the predictor itself remains well-defined for any
user-specified positive learning-rate schedule and deterministic restart
schedule.}

\begingroup
\color{black}
The two order-related exponents have distinct meanings. The condition
$\zeta<1$ requires the time-averaged posterior mass outside the lagged
oracle's stratum to vanish. The condition $b<1$ allows the oracle order to
change, but requires the cumulative cost of those changes to be sublinear;
for restart, $N_T^{\mathrm{ord}}=O(T^s)$ and $H=O(T^h)$ give the concrete
sufficient condition $s+h<1$.
\endgroup

\textcolor{black}{For the unpenalized score and equal-weight component
laws considered here, the sublinear regime essentially requires $p_{\max}$ to
match the generating component count. Allowing larger orders can reduce
clustering risk while worsening distributional approximation. The fixed order
prior does not guarantee recovery of the generating order: excess-order
concentration can make $\Xi_T$ linear, or leave persistent approximation error
if the oracle order is enlarged. The finite-time bounds retain these terms;
matching the cap alone does not establish sublinearity.}

\begingroup
\color{black}
Restart replaces one global comparator by segment-local comparators, yielding
\eqref{eq:restart-clean-term}. Sliding windows and discounting instead localize
the score continuously; under \eqref{eq:localized-memory-control}, the same
polynomial regime applies to all three non-full-history rules.
\endgroup

\begingroup
\color{black}
\section{Experiments}
\label{sec:experiments}

\subsection{Experimental Goals}

\textcolor{black}{The experiments illustrate finite-horizon behavior under
drift and corruption; they do not numerically validate the asymptotic bound.}
We therefore compare the same quasi-Bayesian predictor with and without
restart across increasing horizons, and separately vary the restart interval,
corruption budget, and learning-rate schedule.

\subsection{Experimental Setup}

We use a bounded one-dimensional stream whose clean law is
\[
    p_t^*=\frac12\sum_{j=1}^2
    \operatorname{Unif}[c_{t,j}-0.25,c_{t,j}+0.25].
\]
The two centers are separated by $3$ and undergo
$\lfloor T^a\rfloor$ bounded, non-repeating regime changes. Corruption occurs
at $\lfloor T^\gamma\rfloor$ time indices and displaces each selected clean
observation by at most $\delta$, after which it is clipped to $[-5,5]$. The
learner observes only the corrupted stream. In the main comparison,
$a=0.2$, $\gamma=0.1$, $\delta=1$, $\lambda_t=t^{-0.2}$, and the restart
interval is $H=\operatorname{round}(5T^{0.3})$. We use
$T\in\{500,1000,2000,4000,8000\}$ and ten independent seeds.

At a selected corruption time, the displacement is causal and directional:
letting $\bar x_{t-1}$ be the past clean empirical mean, we set
\[
    \widetilde x_t
    =\operatorname{clip}_{[-5,5]}
    \left(x_t+\delta\,\operatorname{sgn}(x_t-\bar x_{t-1})\right).
\]
Thus the corruption pushes the current sample away from the past center of
mass while using no future data. We save both the total corruption count and
the maximum count in any restart segment. At $T=4000$, the grid
$\gamma\in\{0.1,0.3,0.5,1.0\}$ gives respectively $2$, $12$, $63$, and $4000$
corruptions in total, and at most $1$, $1$, $1$, and $60$ in any restart
segment ($H=60$); these counts do not depend on $\delta$.

For posterior inference, $p_{\max}=2$ and
$q(k)\propto\exp(-0.1k)$.
\textcolor{black}{The principal runs match the cap to the two generating components;
Appendix~\ref{app:model-order-cap} separately tests larger caps.}
Conditional on $k$, centers are uniform on the
admissible bounded interval. RJMCMC proposes $k'\in\{k-1,k,k+1\}$ and uses a
three-degree-of-freedom Student-$t$ proposal around empirical quantile centers.
At each time step we run $200$ iterations, discard the first quarter, and retain
$10$ configurations. Prediction is performed before updating on the current
observation.

Because both the true and predictive laws are finite mixtures of
equal-radius uniform kernels, we compute the one-dimensional distance exactly
as
\[
    W_1(\widehat p_t,p_t^*)
    =\int_{\mathbb R}|F_{\widehat p_t}(x)-F_{p_t^*}(x)|\,\mathrm dx,
\]
without Monte Carlo evaluation noise. Figure~\ref{fig:main-regret-scaling}
reports cumulative regret across horizons; Appendix~\ref{app:additional-experiments}
reports the targeted sensitivity analyses and the separate SPY real-data
evaluation in Appendix~\ref{app:real-data-experiments}.

\begin{figure}[H]
    \centering
    \begin{minipage}{0.48\linewidth}
        \centering
        \textbf{\small (a) Cumulative regret}
        \begin{tikzpicture}
        \begin{axis}[
            theoryplot,
            xlabel={Horizon $T$},
            ylabel={$R_T$},
            xmode=log,
            log basis x=2,
            ymode=log,
            xtick={500,1000,2000,4000,8000},
            xticklabels={500,1000,2000,4000,8000},
            legend pos=north west,
        ]
        \addplot[blue,mark=*,theory error bars]
            table[x=T,y=full_regret,y error minus=full_regret_minus,
            y error plus=full_regret_plus,col sep=comma]
            {\forgettingexperimentdatadir/scaling.csv};
        \addlegendentry{Full history ($\hat\alpha=1.16$)}
        \addplot[red!75!black,mark=*,theory error bars]
            table[x=T,y=restart_regret,y error minus=restart_regret_minus,
            y error plus=restart_regret_plus,col sep=comma]
            {\forgettingexperimentdatadir/scaling.csv};
        \addlegendentry{Restart ($\hat\alpha=0.83$)}
        \addplot[cyan!60!black,mark=triangle*,theory error bars]
            table[x=T,y=sliding_regret,y error minus=sliding_regret_minus,
            y error plus=sliding_regret_plus,col sep=comma]
            {\forgettingexperimentdatadir/scaling.csv};
        \addlegendentry{Sliding ($\hat\alpha=0.84$)}
        \addplot[orange!85!black,mark=square*,theory error bars]
            table[x=T,y=discount_regret,y error minus=discount_regret_minus,
            y error plus=discount_regret_plus,col sep=comma]
            {\forgettingexperimentdatadir/scaling.csv};
        \addlegendentry{Discounted ($\hat\alpha=1.06$)}
        \end{axis}
        \end{tikzpicture}
    \end{minipage}
    \hfill
    \begin{minipage}{0.48\linewidth}
        \centering
        \textbf{\small (b) Normalized regret}
        \begin{tikzpicture}
        \begin{axis}[
            theoryplot,
            xlabel={Horizon $T$},
            ylabel={$R_T/T$},
            xmode=log,
            log basis x=2,
            xtick={500,1000,2000,4000,8000},
            xticklabels={500,1000,2000,4000,8000},
        ]
        \addplot[blue,mark=*,theory error bars]
            table[x=T,y=full_normalized,y error minus=full_normalized_minus,
            y error plus=full_normalized_plus,col sep=comma]
            {\forgettingexperimentdatadir/scaling.csv};
        \addplot[red!75!black,mark=*,theory error bars]
            table[x=T,y=restart_normalized,y error minus=restart_normalized_minus,
            y error plus=restart_normalized_plus,col sep=comma]
            {\forgettingexperimentdatadir/scaling.csv};
        \addplot[cyan!60!black,mark=triangle*,theory error bars]
            table[x=T,y=sliding_normalized,y error minus=sliding_normalized_minus,
            y error plus=sliding_normalized_plus,col sep=comma]
            {\forgettingexperimentdatadir/scaling.csv};
        \addplot[orange!85!black,mark=square*,theory error bars]
            table[x=T,y=discount_normalized,y error minus=discount_normalized_minus,
            y error plus=discount_normalized_plus,col sep=comma]
            {\forgettingexperimentdatadir/scaling.csv};
        \end{axis}
        \end{tikzpicture}
    \end{minipage}
    \caption{\textcolor{black}{Regret scaling under non-repeating distribution shifts. Left:
    cumulative Wasserstein regret $R_T$, with fitted power-law exponent
    $\hat\alpha$. Right: normalized regret $R_T/T$.
    \textcolor{black}{The fit $\hat\alpha<1$ illustrates slower-than-linear
    growth over these horizons, not an asymptotic rate;} 
    the decreasing normalized curve is the corresponding consequence. Error bars are 95\%
    confidence intervals over ten seeds.}}
    \label{fig:main-regret-scaling}
\end{figure}

Restart and sliding have fitted exponents $0.83$ and $0.84$, while full history
and discounting have exponents $1.16$ and $1.06$. The restart guarantee is
derived directly; the sliding-window and discounted conclusions remain
conditional on \eqref{eq:localized-memory-control}.
\textcolor{black}{The bound's certified exponent is about $0.97$,
but its constants make it numerically uninformative at these horizons; the
fitted exponents describe only this finite range. Sliding discards observations
older than $H$, whereas discounting retains them with decaying weight. This may
explain its worse tracking here, but the experiment does not establish the cause.}

\begingroup
\color{black}
To examine the cross-stratum term in the same runs, we additionally compute
$\xi_t=\hat\rho_t(k\ne k_{t-1}^*)$ and
$\Xi_T=\sum_{t=2}^T\xi_t$. Here $k_t^*=2$ throughout, while the RJMCMC
posterior ranges over $k\in\{1,2\}$; hence this experiment isolates posterior
order leakage. The oracle order does not change, so
$N_T^{\mathrm{ord}}=J_T^u=0$ for every memory rule.

\begin{figure}[H]
    \centering
    \begin{minipage}{0.48\linewidth}
        \centering
        \textbf{\small (a) Average off-order mass}
        \begin{tikzpicture}
        \begin{axis}[
            theoryplot,
            xlabel={Horizon $T$},
            ylabel={$\Xi_T/T$},
            xmode=log,
            log basis x=2,
            ymode=log,
            xtick={500,1000,2000,4000,8000},
            xticklabels={500,1000,2000,4000,8000},
            legend to name=orderleakageleftlegend,
            legend columns=2,
            legend style={font=\scriptsize,draw=none,fill=none,
                /tikz/every even column/.append style={column sep=0.7em}},
        ]
        \addplot[blue,mark=*,theory error bars]
            table[x=T,y=full_average,y error=full_average_err,col sep=comma]
            {\forgettingexperimentdatadir/order_leakage_scaling.csv};
        \addlegendentry{Full history}
        \addplot[red!75!black,mark=*,theory error bars]
            table[x=T,y=restart_average,y error=restart_average_err,col sep=comma]
            {\forgettingexperimentdatadir/order_leakage_scaling.csv};
        \addlegendentry{Restart}
        \addplot[cyan!60!black,mark=triangle*,theory error bars]
            table[x=T,y=sliding_average,y error=sliding_average_err,col sep=comma]
            {\forgettingexperimentdatadir/order_leakage_scaling.csv};
        \addlegendentry{Sliding}
        \addplot[orange!85!black,mark=square*,theory error bars]
            table[x=T,y=discount_average,y error=discount_average_err,col sep=comma]
            {\forgettingexperimentdatadir/order_leakage_scaling.csv};
        \addlegendentry{Discounted}
        \end{axis}
        \end{tikzpicture}
    \end{minipage}
    \hfill
    \begin{minipage}{0.48\linewidth}
        \centering
        \textbf{\small (b) Cumulative off-order mass}
        \begin{tikzpicture}
        \begin{axis}[
            theoryplot,
            xlabel={Horizon $T$},
            ylabel={$\Xi_T$},
            xmode=log,
            log basis x=2,
            ymode=log,
            xtick={500,1000,2000,4000,8000},
            xticklabels={500,1000,2000,4000,8000},
            legend to name=orderleakagerightlegend,
            legend columns=2,
            legend style={font=\scriptsize,draw=none,fill=none,
                /tikz/every even column/.append style={column sep=0.7em}},
        ]
        \addplot[gray,dashed]
            table[x=T,y=linear_reference,col sep=comma]
            {\forgettingexperimentdatadir/order_leakage_scaling.csv};
        \addlegendentry{Linear $O(T)$}
        \addplot[blue,mark=*,theory error bars]
            table[x=T,y=full_cumulative,y error=full_cumulative_err,col sep=comma]
            {\forgettingexperimentdatadir/order_leakage_scaling.csv};
        \addlegendentry{Full history ($\hat\zeta=0.02$)}
        \addplot[red!75!black,mark=*,theory error bars]
            table[x=T,y=restart_cumulative,y error=restart_cumulative_err,col sep=comma]
            {\forgettingexperimentdatadir/order_leakage_scaling.csv};
        \addlegendentry{Restart ($\hat\zeta=0.71$)}
        \addplot[cyan!60!black,mark=triangle*,theory error bars]
            table[x=T,y=sliding_cumulative,y error=sliding_cumulative_err,col sep=comma]
            {\forgettingexperimentdatadir/order_leakage_scaling.csv};
        \addlegendentry{Sliding ($\hat\zeta=0.02$)}
        \addplot[orange!85!black,mark=square*,theory error bars]
            table[x=T,y=discount_cumulative,y error=discount_cumulative_err,col sep=comma]
            {\forgettingexperimentdatadir/order_leakage_scaling.csv};
        \addlegendentry{Discounted ($\hat\zeta=0.02$)}
        \end{axis}
        \end{tikzpicture}
    \end{minipage}
    \par\smallskip
    \begin{minipage}{0.48\linewidth}
        \centering
        \ref*{orderleakageleftlegend}
    \end{minipage}
    \hfill
    \begin{minipage}{0.48\linewidth}
        \centering
        \ref*{orderleakagerightlegend}
    \end{minipage}
    \caption{\textcolor{black}{Model-order leakage for the matched cap
    $p_{\max}=k^*=2$, over ten paired seeds. Left: average off-order mass
    $\Xi_T/T$. Right: cumulative off-order mass and finite-range fitted
    exponents. Error bars are 95\% confidence intervals; the dashed line is a
    linear reference. Larger caps are examined in
    Appendix~\ref{app:model-order-cap}.}}
    \label{fig:main-order-leakage}
\end{figure}

The decreasing average leakage is specific to this matched-cap setting;
these finite-horizon fits do not establish asymptotic concentration.
Appendix~\ref{app:model-order-cap} compares $p_{\max}=2,3,4$ with generating
order two and reports both leakage and cumulative Wasserstein error.
\endgroup

\begin{figure}[H]
    \centering
    \begin{minipage}{0.48\linewidth}
        \centering
        \textbf{\small (a) Cumulative regret within one stream}
        \begin{tikzpicture}
        \begin{axis}[
            theoryplot,
            xlabel={Time $t$},
            ylabel={$R_t$},
            xmin=0,xmax=4000,
            xtick={0,1000,2000,3000,4000},
            legend pos=north west,
        ]
        \addplot[blue,thick] table[x=t,y=full_cumulative,col sep=comma]
            {\forgettingexperimentdatadir/online_regret_trajectory_T4000.csv};
        \addlegendentry{Full history}
        \addplot[red!75!black,thick] table[x=t,y=restart_cumulative,col sep=comma]
            {\forgettingexperimentdatadir/online_regret_trajectory_T4000.csv};
        \addlegendentry{Restart}
        \addplot[cyan!60!black,thick] table[x=t,y=sliding_cumulative,col sep=comma]
            {\forgettingexperimentdatadir/online_regret_trajectory_T4000.csv};
        \addlegendentry{Sliding}
        \addplot[orange!85!black,thick] table[x=t,y=discount_cumulative,col sep=comma]
            {\forgettingexperimentdatadir/online_regret_trajectory_T4000.csv};
        \addlegendentry{Discounted}
        \end{axis}
        \end{tikzpicture}
    \end{minipage}
    \hfill
    \begin{minipage}{0.48\linewidth}
        \centering
        \textbf{\small (b) Average regret within one stream}
        \begin{tikzpicture}
        \begin{axis}[
            theoryplot,
            xlabel={Time $t$},
            ylabel={$R_t/t$},
            xmin=0,xmax=4000,
            xtick={0,1000,2000,3000,4000},
        ]
        \addplot[blue,thick] table[x=t,y=full_average,col sep=comma]
            {\forgettingexperimentdatadir/online_regret_trajectory_T4000.csv};
        \addplot[red!75!black,thick] table[x=t,y=restart_average,col sep=comma]
            {\forgettingexperimentdatadir/online_regret_trajectory_T4000.csv};
        \addplot[cyan!60!black,thick] table[x=t,y=sliding_average,col sep=comma]
            {\forgettingexperimentdatadir/online_regret_trajectory_T4000.csv};
        \addplot[orange!85!black,thick] table[x=t,y=discount_average,col sep=comma]
            {\forgettingexperimentdatadir/online_regret_trajectory_T4000.csv};
        \end{axis}
        \end{tikzpicture}
    \end{minipage}
    \caption{\textcolor{black}{Within-stream regret at $T=4000$ in the
    drifting-corrupted setting, averaged over ten paired seeds. Left:
    cumulative Wasserstein regret. Right: running average regret. The quantity
    $R_t/t$ decreases exactly when the current per-step error is below its
    preceding running average; its decline is therefore not a mechanical
    consequence of the denominator.}}
    \label{fig:main-online-regret-trajectory}
\end{figure}

\textcolor{black}{Figure~\ref{fig:main-online-regret-trajectory} complements the
cross-horizon scaling test with the behavior of a single deployed stream.
The cumulative curve records total error accrued by time $t$, whereas $R_t/t$
shows whether recent errors are smaller or larger than the historical average.}

\subsection{Factorial Design and Forgetting Baselines}

We additionally use one $2\times2$ design: stationary-clean,
stationary-corrupted, drifting-clean, and drifting-corrupted. In every setting
we compare full history, hard restart, a sliding window, and exponentially
discounted history. The latent model, prior, stream, learning-rate family, and
RJMCMC budget are identical across methods. Restart and sliding window use the
pre-specified scale $H=\operatorname{round}(5T^{0.3})$; discounting uses
$\rho=\exp(-1/H)$. These values are fixed before examining test outcomes.
Methods within a seed receive exactly the same clean stream and corruption
realization, and all summaries use ten paired seeds with 95\% confidence
intervals.
The resulting comparison is reported in
Appendix~\ref{app:additional-experiments}, Figure~\ref{fig:supp-factorial-baselines}.
\textcolor{black}{The accompanying RJMCMC audit is a computational diagnostic,
not a formal convergence guarantee; in particular, death-move acceptance is
low.}

\subsection{Predictive-Family Approximation}

\textcolor{black}{\textcolor{black}{We isolate predictive-family mismatch
by varying $A_T^{\mathrm{app}}\asymp T^q$ across separate uncorrupted runs,
with a fixed true law for each horizon $T$.} The
results in Appendix~\ref{app:approximation-experiment} show sublinear fitted
growth and decreasing $R_T/T$ for $q<1$, while $q=1$ reaches the linear
boundary, consistently with
Proposition~\ref{prop:sublinear-regime-comparison}.}

\endgroup

\section{Discussion and Conclusion}
\label{sec:discussion}

The main conceptual point of the paper is that the online clustering problem is
best understood through two spaces at once. Learning takes place in the latent
configuration space of cluster-center arrangements, where the prior,
quasi-posterior, and PAC-Bayesian comparison are defined. Prediction, however,
is evaluated in data space through the induced law
$p(\cdot\mid\bc)$. This separation is what makes the method both interpretable
and analyzable: the latent space carries the clustering geometry, while the
data space carries the predictive object that ultimately matters.

From this viewpoint, corruption and drift create two different obstacles.
Corruption perturbs the loss-based posterior reweighting, while drift breaks the
validity of a single global comparator over the whole horizon.
\textcolor{black}{Restart, sliding windows, and discounting localize
memory differently. The restart result is derived directly; the sliding-window
and discounted conclusions use the stated localized-memory condition.} Under the current
bounds, the no-restart predictor is not certified to recover the favorable
long-horizon regime because the clean drift term remains global.

The experiments support this interpretation empirically. They show that
temporal localization improves tracking when the stream continues to move.
At the same time, the experiments also suggest that the asymptotic regime
distinctions in the theory are subtle in finite samples. This is not a
contradiction, but rather an indication that the current bounds capture the
correct structural effect more clearly than they capture sharp finite-sample
constants.

\textcolor{black}{The approximation budget $A_T^{\mathrm{app}}$ separates online
tracking error from finite predictive-family mismatch. Exact realizability is
not required, but sublinear cumulative regret requires
$A_T^{\mathrm{app}}=o(T)$.}

\textcolor{black}{\paragraph{Limitations.}
The guarantees require bounded support and non-adaptive tuning, and depend on
predictive-family approximation and ideal quasi-posterior sampling. Finite
RJMCMC introduces additional error. The SPY study is illustrative and does not
establish trading utility.}

Several directions remain open. On the theoretical side, the corruption
comparison could be sharpened further. \textcolor{black}{A direct derivation of
the localized-memory bound for sliding-window and discounted scores, and a
guarantee for adaptive memory, remain open.} On the algorithmic side, the most natural
next step is to replace the expensive online posterior tracking mechanism by an
amortized or state-space-based inference module while preserving the latent
clustering interpretation. \textcolor{black}{Component weights could also be
learned through a Dirichlet quasi-posterior or discounted online assignment
frequencies.} More broadly, the framework suggests that
quasi-Bayesian latent-geometry methods can be studied as sequential predictive
models rather than only as static clustering procedures. This is the perspective
from which the present paper should be read, and it is the main direction in
which the methodology can be extended.

\bibliographystyle{tmlr}
\bibliography{main}

\clearpage
\appendix
\appendixequations
\appendixfigures
\appendixtheorems

\begin{center}
    {\Large\bfseries Appendix}
\end{center}
\vspace{0.5em}

\section{Proofs of the Main Results}
\label{app:proofs}

\subsection{Proof Organization and Auxiliary Notation}

This appendix is organized in the same logical order as the main paper. We
first prove the corruption-side comparison lemmas for the base update.
\textcolor{black}{For the clean term, the proof then follows four
steps: (i) compare the posterior with the previous oracle; (ii) split posterior
mass into the oracle-order stratum and all other strata; (iii) apply quadratic
growth only to the first part and bound the second by its mass; and (iv) use a
PAC--Bayes bound for the current-risk excess.} We finally combine the corruption
and clean components to obtain the total-regret theorems.

Throughout the appendix, it is convenient to separate the total regret into the
same two pieces used conceptually in Section~\ref{sec:theory}:
\begin{equation}
    \WR^{\mathrm{total}}(T)
    \le
    \Delta_1
    +
    \sum_{t=2}^T
    W_1\!\left(p(x_t\mid \tilde x_{1:t-1}),p(x_t\mid x_{1:t-1})\right)
    +
    \sum_{t=2}^T
    W_1\!\left(p(x_t\mid x_{1:t-1}),p_t^*\right).
    \label{eq:appendix-total-split}
\end{equation}
Here \(\Delta_1:=W_1(p(x_1\mid\emptyset),p_1^*)\le D\). The first sum is the
corruption comparison term, while the second is the clean dynamic term.
Appendix~A.2 develops the first term. Later subsections develop the second term.

For the corruption comparisons, let
\[
    S_{t-1}(c),
    \qquad
    \widetilde S_{t-1}(c)
\]
denote the cumulative scores built from the clean and corrupted histories,
respectively, under the update under consideration. The associated
quasi-posteriors are written as
\begin{equation}
    \rho_t(dc)
    \propto
    \exp\!\big(-\lambda_{t-1}S_{t-1}(c)\big)\,\pi(dc),
    \qquad
    \widetilde\rho_t(dc)
    \propto
    \exp\!\big(-\lambda_{t-1}\widetilde S_{t-1}(c)\big)\,\pi(dc).
    \label{eq:appendix-clean-corrupted-posteriors}
\end{equation}
The proofs below are self-contained within this appendix. In particular,
whenever we refer to an ``earlier'' corruption comparison lemma or proposition,
we mean an earlier statement inside Appendix~A rather than an external result.
For the restarted update, the same display is read segmentwise: if
$t\in I_r$, then $S_{t-1}$ is replaced by the segment-local score
$S_{t-1}^r$ and the global learning rate $\lambda_{t-1}$ is replaced by the
within-segment learning rate $\lambda_{a_r(t)}$.

\subsection{Corruption Comparison Lemmas}

We begin with the three local ingredients needed for the corruption-side
comparative bounds: a cumulative-score perturbation bound, a Gibbs-measure
comparison lemma, and a total-variation-to-predictive-Wasserstein conversion.

\begin{lemma}[Raw loss perturbation under bounded corruption]
    \label{lem:raw-loss-perturbation}
    Let
    \[
        S_{t-1}(c):=\sum_{s=1}^{t-1}\ell(c,x_s),
        \qquad
        \widetilde S_{t-1}(c):=\sum_{s=1}^{t-1}\ell(c,\tilde x_s).
    \]
    If $\ell(c,\cdot)$ is $L_\ell$-Lipschitz and the corruption amplitude is
    bounded by $\delta$, then
    \begin{equation}
        \big|S_{t-1}(c)-\widetilde S_{t-1}(c)\big|
        \le
        L_\ell\delta\,\Lambda_{t-1}
        \label{eq:appendix-raw-score-perturbation}
    \end{equation}
    uniformly over $c$.
\end{lemma}

\begin{proof}
    On any uncorrupted round, $\tilde x_s=x_s$ and the two losses coincide. On
    a corrupted round,
    \[
        \big|\ell(c,x_s)-\ell(c,\tilde x_s)\big|
        \le
        L_\ell \snorm{x_s-\tilde x_s}
        \le
        L_\ell\delta.
    \]
    Summing over the $\Lambda_{t-1}$ corrupted rounds yields
    \eqref{eq:appendix-raw-score-perturbation}.
\end{proof}

\begin{lemma}[Gibbs comparison from energy perturbation]
    \label{lem:gibbs-energy-perturbation}
    Suppose
    \[
        \sup_c
        \big|S_{t-1}(c)-\widetilde S_{t-1}(c)\big|
        \le
        B_{t-1}.
    \]
    Then the quasi-posteriors in
    \eqref{eq:appendix-clean-corrupted-posteriors} satisfy
    \begin{equation}
        d_{\mathrm{TV}}(\rho_t,\widetilde\rho_t)
        \le
        \frac{1}{2}\left[
        \exp\!\big(2\lambda_{t-1}B_{t-1}\big)-1
        \right].
        \label{eq:appendix-gibbs-tv-bound}
    \end{equation}
\end{lemma}

\begin{proof}
    The uniform score bound implies
    \[
        e^{-\lambda_{t-1}B_{t-1}}
        \le
        \exp\!\big(-\lambda_{t-1}(S_{t-1}(c)-\widetilde S_{t-1}(c))\big)
        \le
        e^{\lambda_{t-1}B_{t-1}}
    \]
    for all $c$. After normalizing the two Gibbs densities, this yields a
    density-ratio bound between $\rho_t$ and $\widetilde\rho_t$ controlled by
    $e^{\pm 2\lambda_{t-1}B_{t-1}}$. The standard conversion from a uniform
    density-ratio bound to total variation gives
    \eqref{eq:appendix-gibbs-tv-bound}.
\end{proof}

\begin{lemma}[Predictive Wasserstein comparison from total variation]
    \label{lem:predictive-tv-to-wasserstein}
    Assume that all predictive laws induced by the latent configurations are
    supported in a common set of diameter $D$. Then
    \begin{equation}
        W_1\!\left(p(x_t\mid \tilde x_{1:t-1}),p(x_t\mid x_{1:t-1})\right)
        \le
        D\,d_{\mathrm{TV}}(\rho_t,\widetilde\rho_t).
        \label{eq:appendix-predictive-tv-w1}
    \end{equation}
\end{lemma}

\begin{proof}
    Couple the two predictive laws by first matching the common part of the two
    quasi-posteriors and then paying diameter $D$ on the unmatched remainder.
    For probability measures supported on a set of diameter $D$, this gives
    \[
        W_1(\mu,\nu)\le D\,d_{\mathrm{TV}}(\mu,\nu).
    \]
    Applying this after pushing $\rho_t$ and $\widetilde\rho_t$ through the
    predictive map yields \eqref{eq:appendix-predictive-tv-w1}.
\end{proof}

\begin{proposition}[Raw corruption comparative bound]
    \label{prop:appendix-raw-corruption}
    For the raw quasi-Bayesian update,
    \begin{equation}
        \sum_{t=1}^T
        W_1\!\left(p(x_t\mid \tilde x_{1:t-1}),p(x_t\mid x_{1:t-1})\right)
        \le
        \frac{D}{2}\sum_{t=1}^T
        \left[
        \exp\!\left(
        2\lambda_{t-1}L_\ell\delta\,\Lambda_{t-1}
        \right)-1
        \right].
        \label{eq:appendix-raw-corruption-bound}
    \end{equation}
\end{proposition}

\begin{proof}
    Since the first-order raw score is the cumulative clustering loss,
    Lemma~\ref{lem:raw-loss-perturbation} gives the uniform score perturbation
    \[
        B_{t-1}
        :=
        L_\ell\delta\,\Lambda_{t-1}.
    \]
    Lemma~\ref{lem:gibbs-energy-perturbation} therefore gives
    \[
        d_{\mathrm{TV}}(\rho_t,\widetilde\rho_t)
        \le
        \frac{1}{2}\left[
        \exp\!\big(2\lambda_{t-1}B_{t-1}\big)-1
        \right].
    \]
    Applying Lemma~\ref{lem:predictive-tv-to-wasserstein} yields the
    corresponding one-step predictive Wasserstein comparison, and summing over
    $t=1,\dots,T$ proves \eqref{eq:appendix-raw-corruption-bound}.
\end{proof}

\begin{proposition}[Restarted corruption comparative bound]
    \label{prop:appendix-restart-corruption}
    For the restarted quasi-Bayesian update,
    \begin{equation}
        \sum_{r=1}^{m}
        \sum_{t\in I_r}
        W_1\!\left(p(x_t\mid \tilde x_{\tau_{r-1}:t-1}),p(x_t\mid x_{\tau_{r-1}:t-1})\right)
        \le
        \frac{D}{2}
        \sum_{r=1}^{m}
        \sum_{t\in I_r}
        \left[
        \exp\!\left(
        2\lambda_{a_r(t)}L_\ell\delta\,\Lambda_{r,t-1}
        \right)-1
        \right],
        \label{eq:appendix-restart-corruption-bound}
    \end{equation}
\end{proposition}

\begin{proof}
    Fix a segment \(I_r\) and a time \(t\in I_r\). The clean and corrupted
    restarted scores differ only on corrupted observations inside the current
    segment:
    \[
        S_{t-1}^r(c)
        -
        \widetilde S_{t-1}^r(c)
        =
        \sum_{s=\tau_{r-1}}^{t-1}
        \big[\ell(c,x_s)-\ell(c,\tilde x_s)\big].
    \]
    By the same Lipschitz argument as Lemma~\ref{lem:raw-loss-perturbation},
    \[
        \sup_c
        \big|S_{t-1}^r(c)-\widetilde S_{t-1}^r(c)\big|
        \le
        L_\ell\delta\,\Lambda_{r,t-1}.
    \]
    The restarted posterior at time \(t\) uses the segment-local learning rate
    \(\lambda_{a_r(t)}\). Applying Lemma~\ref{lem:gibbs-energy-perturbation}
    with \(\lambda_{t-1}\) replaced by \(\lambda_{a_r(t)}\) gives
    \[
        d_{\mathrm{TV}}(\rho_t^r,\widetilde\rho_t^r)
        \le
        \frac{1}{2}
        \left[
        \exp\!\left(
        2\lambda_{a_r(t)}L_\ell\delta\,\Lambda_{r,t-1}
        \right)-1
        \right].
    \]
    Lemma~\ref{lem:predictive-tv-to-wasserstein} converts this total-variation
    bound into the corresponding one-step predictive Wasserstein comparison.
    Summing over \(t\in I_r\) and then over segments proves
    \eqref{eq:appendix-restart-corruption-bound}.
\end{proof}

\subsection{\textcolor{black}{Lagged Risk and Wasserstein Drift Lemmas}}

We now turn to the clean dynamic term. The purpose of this subsection is to
establish the bridge from the latent excess-risk analysis to the drift path
$(p_t^*)_{t=1}^T$. The key ingredients are: a Wasserstein control of risk
shifts, a lagged-risk decomposition that explains the constants appearing in
the main text, and a pathwise control of cumulative distribution movement.

{\color{black}
\begin{lemma}[Oracle approximation decomposition]
    \label{lem:appendix-oracle-approximation}
    Let \(V_j(\bc_t^*)\) be the Voronoi cell of \(c_{t,j}^*\), and set
    \(v_{t,j}:=p_t^*(V_j(\bc_t^*))\). For the equal-weight common-kernel map,
    \begin{equation}
        \alpha_t
        \le
        \sqrt{R_t(\bc_t^*)}
        +
        \frac{D}{2}
        \sum_{j=1}^{k_t^*}
        \left|v_{t,j}-\frac{1}{k_t^*}\right|
        +
        \mathbb E\|Z\|_2 .
        \label{eq:appendix-oracle-approximation}
    \end{equation}
    Moreover,
    \[
        \sum_{j=1}^{k_t^*}
        \left|v_{t,j}-\frac1{k_t^*}\right|
        \le 2\left(1-\frac1{k_t^*}\right),
    \]
    so the unequal-mass contribution is at most
    \(D(1-1/k_t^*)<D\).
\end{lemma}

\begin{proof}
    Define
    \[
        \nu_t^V:=\sum_jv_{t,j}\delta_{c_{t,j}^*},
        \qquad
        \nu_t^U:=\frac1{k_t^*}\sum_j\delta_{c_{t,j}^*}.
    \]
    Nearest-center transport gives
    \(W_1(p_t^*,\nu_t^V)\le W_2(p_t^*,\nu_t^V)
    =\sqrt{R_t(\bc_t^*)}\).
    The two discrete laws have common support of diameter at most \(D\), hence
    \[
        W_1(\nu_t^V,\nu_t^U)
        \le
        \frac D2\sum_j\left|v_{t,j}-\frac1{k_t^*}\right|.
    \]
    Finally, coupling each center with the same kernel displacement \(Z\)
    gives
    \(W_1(\nu_t^U,p(\cdot\mid\bc_t^*))\le\mathbb E\|Z\|_2\).
    The triangle inequality proves
    \eqref{eq:appendix-oracle-approximation}.
\end{proof}
}

\begin{lemma}[Risk-shift bound via Kantorovich--Rubinstein]
    \label{lem:appendix-risk-shift}
    For each fixed configuration $c$,
    \begin{equation}
        \big|\bar R_t(c)-R_t(c)\big|
        \le
        4D\,W_1(p_{t-1}^*,p_t^*).
        \label{eq:appendix-risk-shift}
    \end{equation}
\end{lemma}

\begin{proof}
    By definition, the two risks differ only in the probability law with
    respect to which the same loss function is integrated:
    \[
        \bar R_t(c)=\int \ell(c,x)\,dp_{t-1}^*(x),
        \qquad
        R_t(c)=\int \ell(c,x)\,dp_t^*(x).
    \]
    We first recall why the integrand is Lipschitz in the sample variable. Write
    \(c=(c_1,\ldots,c_k)\) and
    \(\ell(c,x)=\min_j\|c_j-x\|_2^2\). For any two points \(x,x'\) in the common
    bounded region,
    \[
        \begin{aligned}
        |\ell(c,x)-\ell(c,x')|
        &=
        \left|
        \min_j\|c_j-x\|_2^2-\min_j\|c_j-x'\|_2^2
        \right|  \\
        &\le
        \max_j
        \left|
        \|c_j-x\|_2^2-\|c_j-x'\|_2^2
        \right| .
        \end{aligned}
    \]
    For each center \(c_j\),
    \[
        \left|
        \|c_j-x\|_2^2-\|c_j-x'\|_2^2
        \right|
        =
        \left|
        \|c_j-x\|_2-\|c_j-x'\|_2
        \right|
        \big(\|c_j-x\|_2+\|c_j-x'\|_2\big).
    \]
    The reverse triangle inequality gives
    \[
        \left|
        \|c_j-x\|_2-\|c_j-x'\|_2
        \right|
        \le
        \|x-x'\|_2,
    \]
    and bounded support gives the uniform bound
    \(\|c_j-x\|_2+\|c_j-x'\|_2\le 4D\) under the loose diameter convention used
    throughout the paper. Hence
    \[
        |\ell(c,x)-\ell(c,x')|
        \le
        4D\,\|x-x'\|_2.
    \]
    Thus \(f_c(x):=\ell(c,x)\) is \(4D\)-Lipschitz. Equivalently,
    \(f_c/(4D)\) is \(1\)-Lipschitz. Applying the
    Kantorovich--Rubinstein dual representation,
    \[
        W_1(P,Q)
        =
        \sup_{\|f\|_{\mathrm{Lip}}\le 1}
        \left|\int f\,dP-\int f\,dQ\right|,
    \]
    with \(P=p_{t-1}^*\), \(Q=p_t^*\), and \(f=f_c/(4D)\), yields
    \[
        \left|
        \int \ell(c,x)\,dp_{t-1}^*(x)
        -
        \int \ell(c,x)\,dp_t^*(x)
        \right|
        \le
        4D\,W_1(p_{t-1}^*,p_t^*),
    \]
    which is exactly \eqref{eq:appendix-risk-shift}.
\end{proof}

\begin{lemma}[Pathwise drift bound]
    \label{lem:appendix-pathwise-drift}
    Let
    \[
        A_T^{\mathrm{OT}}
        :=
        \sum_{t=2}^T W_2^2(p_{t-1}^*,p_t^*).
    \]
    Then
    \begin{equation}
        \sum_{t=2}^T W_1(p_{t-1}^*,p_t^*)
        \le
        \sqrt{(T-1)A_T^{\mathrm{OT}}}.
        \label{eq:appendix-pathwise-drift}
    \end{equation}
\end{lemma}

\begin{proof}
    Since $W_1\le W_2$,
    \[
        \sum_{t=2}^T W_1(p_{t-1}^*,p_t^*)
        \le
        \sum_{t=2}^T W_2(p_{t-1}^*,p_t^*)
        =
        \sum_{t=2}^T \sqrt{W_2^2(p_{t-1}^*,p_t^*)}.
    \]
    Applying Cauchy--Schwarz gives
    \[
        \sum_{t=2}^T \sqrt{W_2^2(p_{t-1}^*,p_t^*)}
        \le
        \sqrt{(T-1)\sum_{t=2}^T W_2^2(p_{t-1}^*,p_t^*)}
        =
        \sqrt{(T-1)A_T^{\mathrm{OT}}},
    \]
    which proves \eqref{eq:appendix-pathwise-drift}.
\end{proof}

\begin{proposition}[Lagged-risk decomposition]
    \label{prop:appendix-lagged-risk}
    Define
    \[
        \bar E_t
        :=
        \mathbb E_{c\sim\hat\rho_t}
        \big[\bar R_t(c)-\bar R_t(c_{t-1}^*)\big],
    \]
    and
    \[
        E_t^{\mathrm{cur}}
        :=
        \mathbb E_{c\sim\hat\rho_t}
        \big[R_t(c)-R_t(c_t^*)\big].
    \]
    \textcolor{black}{Let
    \[
        J_t^{\mathrm{lag}}:=\big[R_t(c_t^*)-R_t(c_{t-1}^*)\big]_+.
    \]}
    Then
    \begin{equation}
        \bar E_t
        \le
        E_t^{\mathrm{cur}}
        +
        \textcolor{black}{J_t^{\mathrm{lag}}}
        +
        8D\,W_1(p_{t-1}^*,p_t^*)
        \label{eq:appendix-lagged-risk-decomposition}
    \end{equation}
    and therefore
    \begin{equation}
        \bar E_t
        \le
        E_t^{\mathrm{cur}}
        +
        \textcolor{black}{J_t^{\mathrm{lag}}}
        +
        8D\,W_2(p_{t-1}^*,p_t^*).
        \label{eq:appendix-lagged-risk-8D}
    \end{equation}
\end{proposition}

\begin{proof}
    \textcolor{black}{We first add and subtract the two current-risk
    terms $R_t(c)$ and $R_t(c_t^*)$. This separates learner error, oracle
    movement, and the change from yesterday's law to today's law:}
    \[
        \bar R_t(c)-\bar R_t(c_{t-1}^*)
        =
        \big(R_t(c)-R_t(c_t^*)\big)
        +
        \big(R_t(c_t^*)-R_t(c_{t-1}^*)\big)
    \]
    \[
        \qquad
        +
        \big(\bar R_t(c)-R_t(c)\big)
        +
        \big(R_t(c_{t-1}^*)-\bar R_t(c_{t-1}^*)\big).
    \]
    \textcolor{black}{Taking expectation over
    $c\sim\hat\rho_t$ affects only the terms containing $c$ and yields}
    \[
        \bar E_t
        =
        E_t^{\mathrm{cur}}
        +
        \big(R_t(c_t^*)-R_t(c_{t-1}^*)\big)
        +
        \Gamma_t,
    \]
    where
    \[
        \Gamma_t
        :=
        \mathbb E_{c\sim\hat\rho_t}\big[\bar R_t(c)-R_t(c)\big]
        +
        \big(R_t(c_{t-1}^*)-\bar R_t(c_{t-1}^*)\big).
    \]
    \textcolor{black}{The oracle-difference term has no fixed sign
    when $k_t^*\ne k_{t-1}^*$. By the definition of the positive part,}
    \[
        R_t(c_t^*)-R_t(c_{t-1}^*)\le J_t^{\mathrm{lag}}.
    \]
    \textcolor{black}{When the two oracle orders coincide, $c_{t-1}^*$
    is feasible in the stratum minimized over by $c_t^*$, so this difference
    is nonpositive and $J_t^{\mathrm{lag}}=0$.}

    \textcolor{black}{It remains to control $\Gamma_t$. Its first term
    changes the law in the risk of a generic posterior configuration; its
    second term changes the law at the lagged oracle.}
    Lemma~\ref{lem:appendix-risk-shift} bounds each of the two terms in
    $\Gamma_t$ by $4D\,W_1(p_{t-1}^*,p_t^*)$, hence
    \[
        \Gamma_t
        \le
        8D\,W_1(p_{t-1}^*,p_t^*).
    \]
    \textcolor{black}{Combining the preceding two bounds gives
    $\bar E_t\le E_t^{\mathrm{cur}}+J_t^{\mathrm{lag}}+\Gamma_t$ and therefore proves
    \eqref{eq:appendix-lagged-risk-decomposition}.} The constant
    $8D$ here is therefore exactly $4D+4D$, coming from the two risk-shift
    comparisons. Since \(W_1\le W_2\), \eqref{eq:appendix-lagged-risk-8D}
    follows.
\end{proof}

\begingroup
\color{black}
\begin{lemma}[Cross-stratum predictive control]
    \label{lem:appendix-cross-stratum}
    Let
    \[
        \mathcal A_t:=\{k_{t-1}^*\}\times\mathcal C_{k_{t-1}^*},
        \qquad
        \xi_t:=\hat\rho_t(\mathcal A_t^c),
    \]
    and define the unnormalized same-stratum lagged excess
    \[
        \bar E_t^{\mathrm{same}}
        :=
        \int_{\mathcal A_t}
        \big[\bar R_t(\bc)-\bar R_t(\bc_{t-1}^*)\big]
        \,\hat\rho_t(d\bc).
    \]
    Under bounded support, quadratic growth on $\mathcal A_t$, and
    predictive regularity within $\mathcal A_t$,
    \begin{equation}
        \bar E_t^{\mathrm{same}}
        \le
        E_t^{\mathrm{cur}}
        +J_t^{\mathrm{lag}}
        +8D\,W_2(p_{t-1}^*,p_t^*)
        +D^2\xi_t,
        \label{eq:appendix-same-stratum-excess}
    \end{equation}
    and
    \begin{equation}
        W_1(\hat p_t,p_{t-1}^*)
        \le
        \frac{L_K}{\sqrt\mu}
        \sqrt{\bar E_t^{\mathrm{same}}}
        +D\xi_t+\alpha_{t-1}.
        \label{eq:appendix-cross-stratum-predictive}
    \end{equation}
\end{lemma}

\begin{proof}
    Decompose the full lagged excess into its same- and off-stratum parts:
    \[
        \bar E_t
        =
        \bar E_t^{\mathrm{same}}
        +
        \int_{\mathcal A_t^c}
        [\bar R_t(\bc)-\bar R_t(\bc_{t-1}^*)]\,\hat\rho_t(d\bc).
    \]
    Both risks lie in $[0,D^2]$, so their difference is at least $-D^2$.
    Because $\hat\rho_t(\mathcal A_t^c)=\xi_t$, the off-stratum contribution
    therefore obeys
    \[
        \int_{\mathcal A_t^c}
        [\bar R_t(\bc)-\bar R_t(\bc_{t-1}^*)]\,\hat\rho_t(d\bc)
        \ge -D^2\xi_t.
    \]
    Rearranging the decomposition gives
    $\bar E_t^{\mathrm{same}}\le\bar E_t+D^2\xi_t$.
    Proposition~\ref{prop:appendix-lagged-risk} then bounds $\bar E_t$ by the
    current-risk excess, oracle-order switch cost, and distributional drift,
    proving
    \eqref{eq:appendix-same-stratum-excess}.

    For the predictive bound, convexity and the triangle inequality give
    \[
        W_1(\hat p_t,p_{t-1}^*)
        \le
        \int_{\mathcal A_t}
        W_1(p(\cdot\mid\bc),p(\cdot\mid\bc_{t-1}^*))\,\hat\rho_t(d\bc)
    \]
    \[
        \qquad +(1-\xi_t)\alpha_{t-1}
        +\int_{\mathcal A_t^c}
        W_1(p(\cdot\mid\bc),p_{t-1}^*)\,\hat\rho_t(d\bc).
    \]
    This display is the key separation: only its first integral compares two
    configurations and therefore invokes $d_k$. On $\mathcal A_t$, apply
    predictive regularity and then
    Cauchy--Schwarz:
    \[
        \int_{\mathcal A_t}
        W_1(p(\cdot\mid\bc),p(\cdot\mid\bc_{t-1}^*))\,\hat\rho_t(d\bc)
        \le
        L_K\int_{\mathcal A_t}d_{k_{t-1}^*}(\bc,\bc_{t-1}^*)
        \,\hat\rho_t(d\bc)
    \]
    \[
        \le
        L_K\sqrt{
        \int_{\mathcal A_t}d_{k_{t-1}^*}^2(\bc,\bc_{t-1}^*)
        \,\hat\rho_t(d\bc)}
        \le
        \frac{L_K}{\sqrt\mu}\sqrt{\bar E_t^{\mathrm{same}}},
    \]
    where the last step uses quadratic growth and the fact that the restricted
    posterior mass is at most one. On $\mathcal A_t^c$, no $d_k$ comparison is
    made. Instead, the common support diameter gives
    $W_1(p(\cdot\mid\bc),p_{t-1}^*)\le D$, contributing at most $D\xi_t$.
    Finally, the comparison through the lagged oracle predictive law costs
    $W_1(p(\cdot\mid\bc_{t-1}^*),p_{t-1}^*)=\alpha_{t-1}$,
    which proves \eqref{eq:appendix-cross-stratum-predictive}.
\end{proof}

\begin{proposition}[Gibbs-evidence control of model-order leakage]
    \label{prop:appendix-order-evidence}
    For the score used to form $\hat\rho_t$, define
    \[
        Z_{t,k}
        :=
        \int_{\mathcal C_k}
        \exp\!\big(-\lambda_{t-1}S_{t-1}(\bc)\big)\,\pi_k(d\bc)
    \]
    and
    \[
        G_t^{\mathrm{ord}}
        :=
        \log
        \frac{q(k_{t-1}^*)Z_{t,k_{t-1}^*}}
        {\max_{k\ne k_{t-1}^*}q(k)Z_{t,k}}.
    \]
    For $p_{\max}>1$,
    \begin{equation}
        \xi_t
        \le
        (p_{\max}-1)\exp(-G_t^{\mathrm{ord}}).
        \label{eq:appendix-order-evidence-bound}
    \end{equation}
    Hence, if for some $r>0$,
    \[
        G_t^{\mathrm{ord}}
        \ge r\log t+\log(p_{\max}-1),
    \]
    then $\xi_t\le t^{-r}$ and
    \[
        \Xi_T=
        \begin{cases}
            O(T^{1-r}), & 0<r<1,\\
            O(\log T), & r=1,\\
            O(1), & r>1.
        \end{cases}
    \]
    More primitively, let $U_t\subset\mathcal C_{k_{t-1}^*}$ satisfy
    $\pi_{k_{t-1}^*}(U_t)>0$ and define the active-score separation
    \[
        \Delta_t^{\mathrm{ord}}(U_t)
        :=
        \min_{k\ne k_{t-1}^*}\inf_{\bc\in\mathcal C_k}S_{t-1}(\bc)
        -\sup_{\bc\in U_t}S_{t-1}(\bc).
    \]
    Then
    \begin{equation}
        G_t^{\mathrm{ord}}
        \ge
        \lambda_{t-1}\Delta_t^{\mathrm{ord}}(U_t)
        +\log
        \frac{q(k_{t-1}^*)\pi_{k_{t-1}^*}(U_t)}
        {\max_{k\ne k_{t-1}^*}q(k)}.
        \label{eq:appendix-order-score-gap}
    \end{equation}
\end{proposition}

\begin{proof}
    The quantity $A_{t,k}:=q(k)Z_{t,k}$ is the unnormalized quasi-posterior
    evidence assigned to order $k$. Thus the posterior probability of order
    $k$ is $A_{t,k}/\sum_jA_{t,j}$. Let
    $A_t^*:=A_{t,k_{t-1}^*}$, and set
    $M_t:=\max_{k\ne k_{t-1}^*}A_{t,k}$. Then
    \[
        \xi_t
        =
        \frac{\sum_{k\ne k_{t-1}^*}A_{t,k}}
        {A_t^*+\sum_{k\ne k_{t-1}^*}A_{t,k}}
        \le
        (p_{\max}-1)\frac{M_t}{A_t^*}
        =
        (p_{\max}-1)e^{-G_t^{\mathrm{ord}}}.
    \]
    The inequality uses only that there are at most $p_{\max}-1$ competing
    orders and each has evidence at most $M_t$. The stated evidence condition
    gives $\xi_t\le t^{-r}$; summing the corresponding $p$-series proves the
    three rates.

    It remains to connect the evidence gap to a score gap. Restricting the
    oracle-order integral to $U_t$ gives the lower bound
    \[
        Z_{t,k_{t-1}^*}
        \ge
        \pi_{k_{t-1}^*}(U_t)
        \exp\!\left(-\lambda_{t-1}
        \sup_{\bc\in U_t}S_{t-1}(\bc)\right),
    \]
    because every score in $U_t$ is at most its displayed supremum. For a
    competing order $k$,
    \[
        Z_{t,k}
        \le
        \exp\!\left(-\lambda_{t-1}
        \inf_{\bc\in\mathcal C_k}S_{t-1}(\bc)\right).
    \]
    because the prior $\pi_k$ has total mass one. Dividing the first bound by
    the largest competing bound, multiplying by the order-prior masses, and
    taking logarithms proves \eqref{eq:appendix-order-score-gap}.
\end{proof}

The evidence condition is sufficient, not automatic. If
$G_t^{\mathrm{ord}}\le0$, a competing order has at least as much Gibbs
evidence as the oracle order. If $G_t^{\mathrm{ord}}=O(1)$, vanishing leakage
is likewise not guaranteed. In particular, if
$\Delta_t^{\mathrm{ord}}(U_t)=0$, the score supplies no separation and only
the fixed prior and neighborhood-mass factors remain; this does not guarantee
decaying leakage. Positive separation is useful only when
$\lambda_{t-1}\Delta_t^{\mathrm{ord}}(U_t)$ grows enough to dominate those
complexity factors and $\log t$. For restarted or localized scores this condition
is interpreted using the active score; temporary leakage after memory resets
or order changes is permitted by the cumulative quantity $\Xi_T$.
\textcolor{black}{This is a conditional evidence bound, not an
order-consistency result. Adding centers cannot increase the optimized
clustering loss, and a fixed order prior need not prevent excess-order
concentration. Sublinear leakage and predictive-family approximation therefore
remain separate requirements.}
\endgroup

\begingroup
\color{black}
\begin{corollary}[Learning-to-yesterday bound]
    \label{cor:appendix-learning-to-yesterday}
    Under the quadratic-growth and predictive-regularity assumptions,
    \begin{equation}
      \begin{aligned}
        \sum_{t=2}^T W_1(\hat p_t,p_{t-1}^*)
        &\le
        \frac{L_K}{\sqrt{\mu}}
        \sqrt{
        (T-1)\left(
        E_T^{\mathrm{cur}}
        +
        8D\sqrt{(T-1)A_T^{\mathrm{OT}}}
        +J_T^{\mathrm{lag}}
        +D^2\Xi_T(\hat\rho)
        \right)
        }
        \\
        &\quad +D\Xi_T(\hat\rho)
        +A_T^{\mathrm{app}},
      \end{aligned}
        \label{eq:appendix-learning-to-yesterday}
    \end{equation}
    where
    \[
        E_T^{\mathrm{cur}}
        :=
        \sum_{t=2}^T
        \mathbb E_{c\sim\hat\rho_t}[R_t(c)-R_t(c_t^*)].
    \]
\end{corollary}

\begin{proof}
    Apply Lemma~\ref{lem:appendix-cross-stratum} at each time and sum. By
    Cauchy--Schwarz and \eqref{eq:appendix-same-stratum-excess},
    \[
      \begin{aligned}
        \sum_{t=2}^T W_1(\hat p_t,p_{t-1}^*)
        &\le
        \frac{L_K}{\sqrt{\mu}}
        \sqrt{
        (T-1)\sum_{t=2}^T
        \left(
        E_t^{\mathrm{cur}}
        +
        8D\,W_2(p_{t-1}^*,p_t^*)
        +J_t^{\mathrm{lag}}
        +D^2\xi_t
        \right)}
        \\
        &\quad +D\Xi_T(\hat\rho)
        +\sum_{t=2}^T\alpha_{t-1}.
      \end{aligned}
    \]
    In this sum, $\sum_{t=2}^T J_t^{\mathrm{lag}}=J_T^{\mathrm{lag}}$ and
    $\sum_{t=2}^T\xi_t=\Xi_T(\hat\rho)$. Moreover,
    Lemma~\ref{lem:appendix-pathwise-drift} controls
    $\sum_{t=2}^T W_2(p_{t-1}^*,p_t^*)$ by
    $\sqrt{(T-1)A_T^{\mathrm{OT}}}$, and
    \(\sum_{t=2}^T\alpha_{t-1}\le
    A_T^{\mathrm{app}}\), proving
    \eqref{eq:appendix-learning-to-yesterday}.
\end{proof}
\endgroup

\subsection{Clean PAC-Bayes Excess-Risk Bound}
\label{app:clean-pac-bayes}

We next record the clean high-probability PAC-Bayes inequality used to control
the moving-oracle excess-risk term $E_T^{\mathrm{cur}}$. The deterministic
exponential-weights part contributes the usual complexity term
$\mathrm{KL}(\nu\|\pi)/\lambda$ and the bounded-loss second-order term
proportional to $\sum_t\lambda_t$. The additional confidence terms come from a
uniform PAC-Bayes concentration step for the comparator and a martingale
concentration step for the changing learner posterior. These terms appear in
the analysis, not in the posterior score.

\begin{lemma}[High-probability uniform online PAC-Bayes bound]
    \label{lem:appendix-fixed-comparator-pac-bayes}
    Let $I=\{\tau,\dots,\tau+n-1\}$ be a time interval of length $n$, and let
    $\mathcal F_t$ denote the clean-stream filtration up to time $t$. Assume
    that $0\le \ell(\bc,x_t)\le D^2$ and
    \[
        \mathbb E[\ell(\bc,x_t)\mid \mathcal F_{t-1}]
        =
        R_t(\bc)
    \]
    for every deterministic configuration $\bc$. On this interval, define
    \[
        L_t(\bc):=\sum_{u=\tau}^{t}\ell(\bc,x_u),
        \qquad
        L_{\tau-1}(\bc):=0,
    \]
    and let the first-order Gibbs update used before observing $x_t$, with
    $t=\tau+s$, be
    \[
        \hat\rho_t(d\bc)
        =
        \frac{
        \exp\!\big(-\lambda_s L_{t-1}(\bc)\big)\pi(d\bc)
        }{
        \int_{\mathcal C}\exp\!\big(-\lambda_s L_{t-1}(\bc')\big)\pi(d\bc')
        },
    \]
    where $\lambda_0,\dots,\lambda_{n-1}$ is non-increasing. Then, for every
    $\delta_0\in(0,1)$ and every $\eta>0$, with probability at least
    $1-\delta_0$, simultaneously for all comparator distributions
    $\nu\ll\pi$,
    \begin{equation}
        \sum_{s=0}^{n-1}
        \mathbb E_{\bc\sim\hat\rho_{\tau+s}}
        R_{\tau+s}(\bc)
        \le
        \sum_{s=0}^{n-1}
        \mathbb E_{\bc\sim\nu}R_{\tau+s}(\bc)
        +
        \frac{\mathrm{KL}(\nu\,\|\,\pi)}{\lambda_{n-1}}
        +
        \frac{D^4}{2}\sum_{s=0}^{n-1}\lambda_s
        +
        \frac{\mathrm{KL}(\nu\,\|\,\pi)+\log(2/\delta_0)}{\eta}
        +
        \frac{\eta D^4 n}{8}
        +
        D^2\sqrt{\frac{n}{2}\log\frac{2}{\delta_0}} .
        \label{eq:appendix-fixed-comparator-pac-bayes}
    \end{equation}
\end{lemma}

\begin{proof}
    First fix the realized clean loss sequence and write
    $\ell_t(\bc):=\ell(\bc,x_t)$. The deterministic exponential-weights
    argument gives a regret inequality for realized losses. For
    $\lambda>0$, define the Gibbs potential
    \[
        \Phi_t(\lambda)
        :=
        -\frac{1}{\lambda}
        \log\int \exp\!\big(-\lambda L_t(\bc)\big)\pi(d\bc).
    \]
    Hoeffding's lemma for losses in $[0,D^2]$ gives, for
    $t=\tau+s$,
    \[
        \mathbb E_{\bc\sim\hat\rho_t}\ell_t(\bc)
        \le
        \Phi_t(\lambda_s)-\Phi_{t-1}(\lambda_s)
        +
        \frac{D^4}{2}\lambda_s ,
    \]
    where the constant $D^4/2$ is the deliberately loose bounded-loss
    constant used throughout the paper. Since the learning rates are
    non-increasing and losses are nonnegative, the potentials telescope by the
    usual monotonicity argument:
    \[
        \sum_{s=0}^{n-1}
        \big(\Phi_{\tau+s}(\lambda_s)-\Phi_{\tau+s-1}(\lambda_s)\big)
        \le
        \Phi_{\tau+n-1}(\lambda_{n-1}).
    \]
    The Gibbs variational inequality then gives
    \[
        \Phi_{\tau+n-1}(\lambda_{n-1})
        \le
        \mathbb E_{\bc\sim\nu}L_{\tau+n-1}(\bc)
        +
        \frac{\mathrm{KL}(\nu\,\|\,\pi)}{\lambda_{n-1}}.
    \]
    Combining these displays yields the deterministic realized-loss bound
    \[
        \sum_{s=0}^{n-1}
        \mathbb E_{\bc\sim\hat\rho_{\tau+s}}\ell_{\tau+s}(\bc)
        \le
        \sum_{s=0}^{n-1}
        \mathbb E_{\bc\sim\nu}\ell_{\tau+s}(\bc)
        +
        \frac{\mathrm{KL}(\nu\,\|\,\pi)}{\lambda_{n-1}}
        +
        \frac{D^4}{2}\sum_{s=0}^{n-1}\lambda_s .
    \]
    It remains to convert realized losses to population risks with high
    probability. For the learner side define
    \[
        Y_t^\rho
        :=
        \mathbb E_{\bc\sim\hat\rho_t}\ell(\bc,x_t),
        \qquad
        Z_t^\rho
        :=
        \mathbb E[Y_t^\rho\mid\mathcal F_{t-1}]-Y_t^\rho .
    \]
    Because $\hat\rho_t$ is formed before observing $x_t$, it is
    $\mathcal F_{t-1}$-measurable, and
    $\mathbb E[Y_t^\rho\mid\mathcal F_{t-1}]
    =\mathbb E_{\bc\sim\hat\rho_t}R_t(\bc)$. Moreover
    $Y_t^\rho\in[0,D^2]$, so conditionally $Z_t^\rho$ has range length at most
    $D^2$. Hoeffding--Azuma therefore gives, with probability at least
    $1-\delta_0/2$,
    \[
        \sum_{s=0}^{n-1}
        \mathbb E_{\bc\sim\hat\rho_{\tau+s}}R_{\tau+s}(\bc)
        \le
        \sum_{s=0}^{n-1}
        \mathbb E_{\bc\sim\hat\rho_{\tau+s}}\ell_{\tau+s}(\bc)
        +
        D^2\sqrt{\frac{n}{2}\log\frac{2}{\delta_0}} .
    \]

    For the comparator side, for a fixed configuration $\bc$ define
    $Z_t^{\bc}:=\ell_t(\bc)-R_t(\bc)$. The conditional range length is again at
    most $D^2$, and hence
    \[
        \mathbb E\exp\!\left\{
        \eta\sum_{s=0}^{n-1} Z_{\tau+s}^{\bc}
        \right\}
        \le
        \exp\!\left\{\frac{\eta^2D^4n}{8}\right\}.
    \]
    Integrating this display with respect to the prior $\pi$ and applying
    Markov's inequality gives, with probability at least $1-\delta_0/2$,
    \[
        \int
        \exp\!\left\{
        \eta\sum_{s=0}^{n-1} Z_{\tau+s}^{\bc}
        \right\}\pi(d\bc)
        \le
        \frac{2}{\delta_0}
        \exp\!\left\{\frac{\eta^2D^4n}{8}\right\}.
    \]
    On this event, the Donsker--Varadhan variational inequality implies that
    simultaneously for all $\nu\ll\pi$,
    \[
        \sum_{s=0}^{n-1}
        \mathbb E_{\bc\sim\nu}\ell_{\tau+s}(\bc)
        \le
        \sum_{s=0}^{n-1}
        \mathbb E_{\bc\sim\nu}R_{\tau+s}(\bc)
        +
        \frac{\mathrm{KL}(\nu\,\|\,\pi)+\log(2/\delta_0)}{\eta}
        +
        \frac{\eta D^4n}{8}.
    \]
    A union bound over the learner and comparator concentration events,
    followed by substitution into the deterministic realized-loss inequality,
    proves \eqref{eq:appendix-fixed-comparator-pac-bayes}.
\end{proof}

\begin{lemma}[Localized high-probability comparator consequence]
    \label{lem:appendix-localized-comparator}
    Let $I=\{\tau,\dots,\tau+n-1\}$ and let $\bc^\dagger$ be a fixed
    configuration used as comparator on this interval.
    \textcolor{black}{The constant $L_c=2D$ is valid in
    $|R_t(\bc)-R_t(\bc')|\le L_c d_k(\bc,\bc')$ within a fixed order stratum;
    localization within radius $\varepsilon$ therefore costs at most $L_c\varepsilon$.}
    Suppose there exists a
    localized comparator distribution $\nu_I\ll\pi$ such that
    \[
        \mathrm{KL}(\nu_I\,\|\,\pi)\le C
    \]
    and, for every $t\in I$,
    \[
        \mathbb E_{\bc\sim\nu_I}R_t(\bc)
        \le
        R_t(\bc^\dagger)+L_c\varepsilon .
    \]
    Then, for every $\delta_0\in(0,1)$ and $\eta>0$, with probability at least
    $1-\delta_0$,
    \begin{equation}
        \begin{aligned}
        \sum_{t\in I}
        \mathbb E_{\bc\sim\hat\rho_t}
        \big[R_t(\bc)-R_t(\bc_t^*)\big]
        \le\;&
        L_c n\varepsilon
        +
        \frac{C}{\lambda_{n-1}}
        +
        \frac{D^4}{2}\sum_{s=0}^{n-1}\lambda_s
        +
        \frac{C+\log(2/\delta_0)}{\eta}
        \\
        &+
        \frac{\eta D^4 n}{8}
        +
        D^2\sqrt{\frac{n}{2}\log\frac{2}{\delta_0}}
        +
        \sum_{t\in I}\big[R_t(\bc^\dagger)-R_t(\bc_t^*)\big].
        \end{aligned}
        \label{eq:appendix-localized-comparator}
    \end{equation}
    For the cluster-center prior used in the main text, the assumption
    $\mathrm{KL}(\nu_I\,\|\,\pi)\le C$ is precisely where the latent
    dimension enters the bound; localizing $k$ centers in $\RR^d$ at radius
    $r$ gives the representative scaling
    $C\lesssim \log(1/q(k))+kd\log(2R/r)$, where \(R\) is the
    bounded-support radius. \textcolor{black}{This localization uses $d_k$ within
    the comparator's fixed-$k$ stratum.}
\end{lemma}

\begin{proof}
    Subtract the deterministic quantity $\sum_{t\in I}R_t(\bc_t^*)$ from both
    sides of Lemma~\ref{lem:appendix-fixed-comparator-pac-bayes}. The
    localized comparator assumption gives
    \[
        \sum_{t\in I}
        \mathbb E_{\bc\sim\nu_I}
        \big[R_t(\bc)-R_t(\bc_t^*)\big]
        \le
        L_cn\varepsilon
        +
        \sum_{t\in I}\big[R_t(\bc^\dagger)-R_t(\bc_t^*)\big],
    \]
    which gives \eqref{eq:appendix-localized-comparator}.
\end{proof}

\begin{lemma}[Restarted high-probability PAC-Bayes bound]
    \label{lem:appendix-restarted-pac-bayes}
    Let \(I_1,\ldots,I_m\) be deterministic restart segments, with
    \(I_r=\{\tau_{r-1},\ldots,\tau_r-1\}\), length \(H_r\le H\), and
    \(\sum_{r=1}^m H_r=T\). Let
    \(\lambda_0,\ldots,\lambda_{H-1}\) be non-increasing, and let the restarted
    posterior on segment \(r\) be
    \[
        \hat\rho_t^r(d\bc)
        =
        \frac{
        \exp\!\big(-\lambda_{a_r(t)}L_{t-1}^r(\bc)\big)\pi(d\bc)
        }{
        \int_{\mathcal C}
        \exp\!\big(-\lambda_{a_r(t)}L_{t-1}^r(\bc')\big)\pi(d\bc')
        },
        \qquad
        L_{t-1}^r(\bc):=\sum_{u=\tau_{r-1}}^{t-1}\ell(\bc,x_u).
    \]
    Suppose that, for each segment \(r\), there is a comparator
    \(\bc_r^\dagger\) and a distribution \(\nu_r\ll\pi\) such that
    \[
        \mathrm{KL}(\nu_r\,\|\,\pi)\le C,
        \qquad
        \mathbb E_{\bc\sim\nu_r}R_t(\bc)
        \le
        R_t(\bc_r^\dagger)+L_c\varepsilon,
        \qquad t\in I_r.
    \]
    Then, for every \(\delta_0\in(0,1)\) and \(\eta>0\), with probability at
    least \(1-\delta_0\),
    \begin{equation}
        \begin{aligned}
        &\sum_{r=1}^m\sum_{t\in I_r}
        \mathbb E_{\bc\sim\hat\rho_t^r}
        \big[R_t(\bc)-R_t(\bc_t^*)\big]
        \\
        &\qquad\le
        L_cT\varepsilon
        +
        \frac{mC}{\lambda_{H-1}}
        +
        \frac{D^4}{2}m\bar\lambda_H
        +
        \frac{mC+\log(2/\delta_0)}{\eta}
        +
        \frac{\eta D^4T}{8}
        \\
        &\qquad\quad+
        D^2\sqrt{\frac{T}{2}\log\frac{2}{\delta_0}}
        +
        \sum_{r=1}^m\sum_{t\in I_r}
        \big[R_t(\bc_r^\dagger)-R_t(\bc_t^*)\big].
        \end{aligned}
        \label{eq:appendix-restarted-pac-bayes}
    \end{equation}
\end{lemma}

\begin{proof}
    The reset schedule is handled segment by segment. For a fixed realized
    clean sample path, the deterministic exponential-weights argument from the
    proof of Lemma~\ref{lem:appendix-fixed-comparator-pac-bayes} applies on
    each segment separately, because
    \(\lambda_0,\ldots,\lambda_{H_r-1}\) is non-increasing inside that segment.
    Hence, for every choice of segment comparators \(\nu_1,\ldots,\nu_m\),
    \[
        \sum_{r=1}^m\sum_{t\in I_r}
        \mathbb E_{\bc\sim\hat\rho_t^r}\ell(\bc,x_t)
        \le
        \sum_{r=1}^m\sum_{t\in I_r}
        \mathbb E_{\bc\sim\nu_r}\ell(\bc,x_t)
        +
        \sum_{r=1}^m\frac{\mathrm{KL}(\nu_r\,\|\,\pi)}{\lambda_{H_r-1}}
        +
        \frac{D^4}{2}\sum_{r=1}^m\sum_{s=0}^{H_r-1}\lambda_s .
    \]
    Since \(H_r\le H\) and the schedule is non-increasing,
    \(\lambda_{H_r-1}\ge \lambda_{H-1}\) and
    \(\sum_{s=0}^{H_r-1}\lambda_s\le \bar\lambda_H\). Therefore
    \[
        \sum_{r=1}^m\sum_{t\in I_r}
        \mathbb E_{\bc\sim\hat\rho_t^r}\ell(\bc,x_t)
        \le
        \sum_{r=1}^m\sum_{t\in I_r}
        \mathbb E_{\bc\sim\nu_r}\ell(\bc,x_t)
        +
        \frac{mC}{\lambda_{H-1}}
        +
        \frac{D^4}{2}m\bar\lambda_H .
    \]
    No potential is telescoped across a restart boundary; this is the step that
    justifies resetting the schedule to \(\lambda_0\).

    We next convert realized losses to population risks. The restarted
    posterior \(\hat\rho_t^r\) is formed before observing \(x_t\), so it is
    \(\mathcal F_{t-1}\)-measurable. Applying Hoeffding--Azuma over the \(T\)
    learner terms gives, with probability at least \(1-\delta_0/2\),
    \[
        \sum_{r=1}^m\sum_{t\in I_r}
        \mathbb E_{\bc\sim\hat\rho_t^r}R_t(\bc)
        \le
        \sum_{r=1}^m\sum_{t\in I_r}
        \mathbb E_{\bc\sim\hat\rho_t^r}\ell(\bc,x_t)
        +
        D^2\sqrt{\frac{T}{2}\log\frac{2}{\delta_0}} .
    \]

    For the comparator side, define the product prior
    \(\Pi:=\otimes_{r=1}^m\pi\) on \(\mathcal C^m\). For
    \(\mathbf c=(\bc_1,\ldots,\bc_m)\), let
    \[
        Z_t^{\mathbf c}
        :=
        \ell(\bc_r,x_t)-R_t(\bc_r),
        \qquad t\in I_r.
    \]
    The variables \(Z_t^{\mathbf c}\) are martingale differences with
    conditional range length at most \(D^2\), so
    \[
        \mathbb E\exp\!\left\{
        \eta\sum_{r=1}^m\sum_{t\in I_r} Z_t^{\mathbf c}
        \right\}
        \le
        \exp\!\left\{\frac{\eta^2D^4T}{8}\right\}.
    \]
    Integrating over \(\Pi\), applying Markov's inequality, and then applying
    the Donsker--Varadhan variational inequality gives, with probability at
    least \(1-\delta_0/2\), simultaneously for all product comparators
    \(N=\otimes_{r=1}^m\nu_r\),
    \[
        \sum_{r=1}^m\sum_{t\in I_r}
        \mathbb E_{\bc\sim\nu_r}\ell(\bc,x_t)
        \le
        \sum_{r=1}^m\sum_{t\in I_r}
        \mathbb E_{\bc\sim\nu_r}R_t(\bc)
        +
        \frac{\sum_{r=1}^m\mathrm{KL}(\nu_r\,\|\,\pi)+\log(2/\delta_0)}{\eta}
        +
        \frac{\eta D^4T}{8}.
    \]
    A union bound over the learner and comparator concentration events, followed
    by \(\sum_r\mathrm{KL}(\nu_r\,\|\,\pi)\le mC\), the localized comparator
    assumption, and subtraction of \(\sum_t R_t(\bc_t^*)\), proves
    \eqref{eq:appendix-restarted-pac-bayes}.
\end{proof}

\subsection{No-Restart Clean Dynamic Bound}
\label{app:no-restart-clean-proof}

We now apply the fixed-comparator inequality to the whole horizon. The
important point is that a single fixed comparator must approximate the entire
moving oracle path.

\begin{lemma}[Whole-horizon moving-oracle mismatch]
    \label{lem:appendix-global-mismatch}
    Let $\bc^\dagger=\bc_1^*$. Then
    \begin{equation}
        \sum_{t=1}^T
        \big[R_t(\bc^\dagger)-R_t(\bc_t^*)\big]
        \le
        8D\,T\sqrt{(T-1)A_T^{\mathrm{OT}}}
        +\textcolor{black}{J_T^{\mathrm{cmp},\mathrm{nr}}}.
        \label{eq:appendix-global-mismatch}
    \end{equation}
\end{lemma}

\begin{proof}
    Fix $t$. Insert the risk at time $1$:
    \[
        R_t(\bc_1^*)-R_t(\bc_t^*)
        =
        \big(R_t(\bc_1^*)-R_1(\bc_1^*)\big)
        +
        \big(R_1(\bc_1^*)-R_1(\bc_t^*)\big)
        +
        \big(R_1(\bc_t^*)-R_t(\bc_t^*)\big).
    \]
    The first and third terms only change the probability law while keeping
    the configuration fixed, so Lemma~\ref{lem:appendix-risk-shift} controls
    them along the path from $1$ to $t$. The middle term is nonpositive when
    $k_t^*=k_1^*$; otherwise its positive part is retained rather than
    discarded. Thus
    \[
        R_t(\bc_1^*)-R_t(\bc_t^*)
        \le
        8D\sum_{s=2}^t W_1(p_{s-1}^*,p_s^*)
        +\big[R_1(\bc_1^*)-R_1(\bc_t^*)\big]_+.
    \]
    Summing over $t$ and bounding the triangular sum by $T$ times the full
    path length yields
    \[
        \sum_{t=1}^T
        \big[R_t(\bc_1^*)-R_t(\bc_t^*)\big]
        \le
        8DT\sum_{s=2}^T W_1(p_{s-1}^*,p_s^*)
        +J_T^{\mathrm{cmp},\mathrm{nr}}.
    \]
    Lemma~\ref{lem:appendix-pathwise-drift} completes the proof.
\end{proof}

\begin{lemma}[No-restart current-risk excess]
    \label{lem:appendix-no-restart-current-risk}
    For the base quasi-Bayesian predictor, with probability at least
    \(1-\delta_0\),
    \begin{equation}
        \begin{aligned}
        E_T^{\mathrm{cur}}
        \le\;&
        L_cT\varepsilon
        +
        \frac{2C+\log(2/\delta_0)}{\lambda_{T-1}}
        +
        \frac{D^4}{2}\bar\lambda_T
        +
        \frac{D^4T\lambda_{T-1}}{8}
        \\
        &+
        D^2\sqrt{\frac{T}{2}\log\frac{2}{\delta_0}}
        +
        8D\,T\sqrt{(T-1)A_T^{\mathrm{OT}}}
        +\textcolor{black}{J_T^{\mathrm{cmp},\mathrm{nr}}}.
        \end{aligned}
        \label{eq:appendix-no-restart-current-risk}
    \end{equation}
\end{lemma}

\begin{proof}
    Apply Lemma~\ref{lem:appendix-localized-comparator} on the whole horizon
    with $\bc^\dagger=\bc_1^*$, $n=T$, and
    $\eta=\lambda_{T-1}$. The localized comparator contributes
    \(L_cT\varepsilon\), the two PAC-Bayes complexity terms combine into
    \((2C+\log(2/\delta_0))/\lambda_{T-1}\), and the remaining
    high-probability terms are inherited directly from
    Lemma~\ref{lem:appendix-localized-comparator}. Finally,
    Lemma~\ref{lem:appendix-global-mismatch} controls the whole-horizon
    moving-oracle mismatch, giving
    \eqref{eq:appendix-no-restart-current-risk}.
\end{proof}

\begin{proposition}[No-restart clean dynamic bound]
    \label{prop:appendix-no-restart-clean}
    For the base quasi-Bayesian predictor, with probability at least
    \(1-\delta_0\),
    \begin{equation}
        \sum_{t=2}^T W_1(\hat p_t,p_t^*)
        \le
        \mathcal G_{T,\delta_0}^{\mathrm{nr}}(\{\lambda_t\}),
        \label{eq:appendix-no-restart-clean}
    \end{equation}
    where $\mathcal G_{T,\delta_0}^{\mathrm{nr}}$ is defined in
    \eqref{eq:no-restart-clean-term}.
\end{proposition}

\begin{proof}
    \textcolor{black}{We proceed in three steps. First,
    Corollary~\ref{cor:appendix-learning-to-yesterday} converts current-risk
    excess into predictive error relative to yesterday's oracle law. It is a
    pathwise statement, so on the event of
    Lemma~\ref{lem:appendix-no-restart-current-risk},}
    \[
        \sum_{t=2}^T W_1(\hat p_t,p_{t-1}^*)
        \le
        \frac{L_K}{\sqrt{\mu}}
        \sqrt{
        (T-1)\left(
        E_T^{\mathrm{cur}}
        +
        8D\sqrt{(T-1)A_T^{\mathrm{OT}}}
        +\textcolor{black}{J_T^{\mathrm{lag}}}
        +\textcolor{black}{D^2\Xi_T(\hat\rho^{\mathrm{nr}})}
        \right)
        }
        +\textcolor{black}{D\Xi_T(\hat\rho^{\mathrm{nr}})}
        +\textcolor{black}{A_T^{\mathrm{app}}}.
    \]
    \textcolor{black}{Second, substitute
    \eqref{eq:appendix-no-restart-current-risk} for
    $E_T^{\mathrm{cur}}$. Its order-change contribution is
    $J_T^{\mathrm{cmp},\mathrm{nr}}$; together with the displayed
    $J_T^{\mathrm{lag}}$ this is precisely $J_T^{\mathrm{nr}}$ by
    \eqref{eq:oracle-comparator-switch-cost}. All remaining terms form
    $\mathcal B_{T,\delta_0}^{\mathrm{nr}}$. Third, move from yesterday's
    oracle law to today's law using}
    \[
        W_1(\hat p_t,p_t^*)
        \le
        W_1(\hat p_t,p_{t-1}^*)
        +
        W_1(p_{t-1}^*,p_t^*)
    \]
    \textcolor{black}{and sum over $t$. By
    Lemma~\ref{lem:appendix-pathwise-drift}, the added path-length term is at
    most $\sqrt{(T-1)A_T^{\mathrm{OT}}}$. These substitutions give exactly
    $\mathcal G_{T,\delta_0}^{\mathrm{nr}}$ in
    \eqref{eq:no-restart-clean-term}.}
\end{proof}

\subsection{Restarted Clean Dynamic Bound}
\label{app:restart-clean-proof}

Restart replaces the single whole-horizon comparator by one local fixed
comparator per segment. This is the only structural difference in the clean
analysis.

\begin{lemma}[Segmentwise moving-oracle mismatch]
    \label{lem:appendix-segment-mismatch}
    Let $I_r=\{\tau_{r-1},\dots,\tau_r-1\}$ have length $H_r$, and define
    \[
        A_r^{\mathrm{OT}}
        :=
        \sum_{t=\tau_{r-1}+1}^{\tau_r-1}
        W_2^2(p_{t-1}^*,p_t^*).
    \]
    If the segment comparator is $\bc_{\tau_{r-1}}^*$, then
    \begin{equation}
        \sum_{t\in I_r}
        \big[R_t(\bc_{\tau_{r-1}}^*)-R_t(\bc_t^*)\big]
        \le
        8D\,H_r^{3/2}\sqrt{A_r^{\mathrm{OT}}}
        +\textcolor{black}{J_{I_r}^{\mathrm{cmp}}},
        \label{eq:appendix-segment-mismatch}
    \end{equation}
    where
    \[
        J_{I_r}^{\mathrm{cmp}}
        :=\sum_{t\in I_r}
        \big[R_{\tau_{r-1}}(\bc_{\tau_{r-1}}^*)
        -R_{\tau_{r-1}}(\bc_t^*)\big]_+.
    \]
\end{lemma}

\begin{proof}
    The proof is the localized version of
    Lemma~\ref{lem:appendix-global-mismatch}. For $t\in I_r$, insert the risk
    at the segment start $\tau_{r-1}$. The middle comparator term is
    nonpositive when the two oracle orders coincide; for a different order we
    retain its positive part in $J_{I_r}^{\mathrm{cmp}}$. The two remaining
    terms are changes of law at a fixed configuration, so
    Lemma~\ref{lem:appendix-risk-shift} gives
    \[
        R_t(\bc_{\tau_{r-1}}^*)-R_t(\bc_t^*)
        \le
        8D
        \sum_{s=\tau_{r-1}+1}^{t}
        W_1(p_{s-1}^*,p_s^*)
        +\big[R_{\tau_{r-1}}(\bc_{\tau_{r-1}}^*)
        -R_{\tau_{r-1}}(\bc_t^*)\big]_+.
    \]
    Summing over the $H_r$ times in the segment gives at most
    \[
        8D H_r
        \sum_{s=\tau_{r-1}+1}^{\tau_r-1}
        W_1(p_{s-1}^*,p_s^*)
        +J_{I_r}^{\mathrm{cmp}}.
    \]
    By $W_1\le W_2$ and Cauchy--Schwarz on the segment,
    \[
        \sum_{s=\tau_{r-1}+1}^{\tau_r-1}
        W_1(p_{s-1}^*,p_s^*)
        \le
        \sqrt{H_r A_r^{\mathrm{OT}}},
    \]
    which proves \eqref{eq:appendix-segment-mismatch}.
\end{proof}

\begin{lemma}[Equal-window segment aggregation]
    \label{lem:appendix-segment-aggregation}
    Suppose $H_r\le H$ for all segments and the number of segments is
    $m=T/H$ for notational simplicity. Then
    \begin{equation}
        \sum_{r=1}^m H_r^{3/2}\sqrt{A_r^{\mathrm{OT}}}
        \le
        H\sqrt{T A_T^{\mathrm{OT}}}.
        \label{eq:appendix-segment-aggregation}
    \end{equation}
    If $H$ does not divide $T$, the same bound holds up to the harmless
    replacement of $T/H$ by $\lceil T/H\rceil$.
\end{lemma}

\begin{proof}
    Since $H_r\le H$,
    \[
        \sum_{r=1}^m H_r^{3/2}\sqrt{A_r^{\mathrm{OT}}}
        \le
        H^{3/2}\sum_{r=1}^m\sqrt{A_r^{\mathrm{OT}}}.
    \]
    Applying Cauchy--Schwarz over segments gives
    \[
        \sum_{r=1}^m\sqrt{A_r^{\mathrm{OT}}}
        \le
        \sqrt{m\sum_{r=1}^m A_r^{\mathrm{OT}}}
        \le
        \sqrt{(T/H)A_T^{\mathrm{OT}}}.
    \]
    Combining the two inequalities yields
    \eqref{eq:appendix-segment-aggregation}.
\end{proof}

\begin{lemma}[Restarted current-risk excess]
    \label{lem:appendix-restart-current-risk}
    For the restarted quasi-Bayesian predictor with deterministic equal-length
    restart window \(H\), assume for the displayed formula that \(H\mid T\).
    Then, with probability at least \(1-\delta_0\),
    \begin{equation}
        \begin{aligned}
        E_T^{\mathrm{cur}}
        \le\;&
        L_cT\varepsilon
        +
        \frac{2(T/H)C+\log(2/\delta_0)}{\lambda_{H-1}}
        +
        \frac{D^4}{2}\frac{T}{H}\bar\lambda_H
        +
        \frac{D^4T\lambda_{H-1}}{8}
        \\
        &+
        D^2\sqrt{\frac{T}{2}\log\frac{2}{\delta_0}}
        +
        8DH\sqrt{TA_T^{\mathrm{OT}}}
        +\textcolor{black}{J_T^{\mathrm{cmp},\mathrm{rs}}}.
        \end{aligned}
        \label{eq:appendix-restart-current-risk}
    \end{equation}
\end{lemma}

\begin{proof}
    Apply Lemma~\ref{lem:appendix-restarted-pac-bayes} with segment comparator
    \(\bc_r^\dagger=\bc_{\tau_{r-1}}^*\), with \(\nu_r\) localized around
    \(\bc_{\tau_{r-1}}^*\), and with \(\eta=\lambda_{H-1}\). For equal-length
    windows with \(H\mid T\), the number of segments is \(m=T/H\), so the two
    KL terms in \eqref{eq:appendix-restarted-pac-bayes} combine as
    \[
        \frac{mC}{\lambda_{H-1}}
        +
        \frac{mC+\log(2/\delta_0)}{\lambda_{H-1}}
        =
        \frac{2(T/H)C+\log(2/\delta_0)}{\lambda_{H-1}}.
    \]
    The bounded-loss terms become
    \[
        \frac{D^4}{2}\frac{T}{H}\bar\lambda_H
        +
        \frac{D^4T\lambda_{H-1}}{8},
    \]
    and the learner concentration term is
    \(D^2\sqrt{(T/2)\log(2/\delta_0)}\). The remaining segmentwise
    moving-oracle mismatch is controlled by
    Lemma~\ref{lem:appendix-segment-aggregation}:
    \[
        8D\sum_{r=1}^{T/H}H_r^{3/2}\sqrt{A_r^{\mathrm{OT}}}
        \le
        8DH\sqrt{TA_T^{\mathrm{OT}}}.
    \]
    \textcolor{black}{Summing the retained segmentwise order terms
    gives $\sum_rJ_{I_r}^{\mathrm{cmp}}=J_T^{\mathrm{cmp},\mathrm{rs}}$.}
    Substituting these quantities gives
    \eqref{eq:appendix-restart-current-risk}. If \(H\nmid T\), the same display
    holds with \(T/H\) replaced by the number of segments
    \(m=\lceil T/H\rceil\).
\end{proof}

\begin{proposition}[Restarted clean dynamic bound]
    \label{prop:appendix-restart-clean}
    For the restarted quasi-Bayesian predictor with equal restart window $H$,
    with probability at least \(1-\delta_0\),
    \begin{equation}
        \sum_{t=2}^T W_1(\hat p_t,p_t^*)
        \le
        \mathcal G_{T,\delta_0}^{\mathrm{rs}}(\{\lambda_s\},H),
        \label{eq:appendix-restart-clean}
    \end{equation}
    where $\mathcal G_{T,\delta_0}^{\mathrm{rs}}$ is defined in
    \eqref{eq:restart-clean-term}.
\end{proposition}

\begin{proof}
    \textcolor{black}{The argument follows the same three steps as the
    no-restart proof, but the PAC--Bayes comparison is now performed segment by
    segment. First, Corollary~\ref{cor:appendix-learning-to-yesterday} gives,
    on the event of Lemma~\ref{lem:appendix-restart-current-risk},}
    \[
        \sum_{t=2}^T W_1(\hat p_t,p_{t-1}^*)
        \le
        \frac{L_K}{\sqrt{\mu}}
        \sqrt{
        T\left(
        E_T^{\mathrm{cur}}
        +
        8D\sqrt{TA_T^{\mathrm{OT}}}
        +\textcolor{black}{J_T^{\mathrm{lag}}}
        +\textcolor{black}{D^2\Xi_T(\hat\rho^{\mathrm{rs}})}
        \right)
        }
        +\textcolor{black}{D\Xi_T(\hat\rho^{\mathrm{rs}})}
        +\textcolor{black}{A_T^{\mathrm{app}}}.
    \]
    \textcolor{black}{Second, substitute
    \eqref{eq:appendix-restart-current-risk}. Its segment-comparator term
    $J_T^{\mathrm{cmp},\mathrm{rs}}$ combines with
    $J_T^{\mathrm{lag}}$ to give $J_T^{\mathrm{rs}}$, while the other terms
    form $\mathcal B_{T,\delta_0}^{\mathrm{rs}}$. Third, apply the triangle
    inequality between $p_{t-1}^*$ and $p_t^*$ and use
    $\sum_{t=2}^T W_1(p_{t-1}^*,p_t^*)\le
    \sqrt{TA_T^{\mathrm{OT}}}$. The result is exactly
    \eqref{eq:restart-clean-term}.}
\end{proof}

\subsection{Proofs of the Total-Regret Theorems}
\label{app:total-theorem-proofs}

\begin{proof}[Proof of Theorem~\ref{thm:raw-total}]
    Start from the decomposition
    \eqref{eq:appendix-total-split}. The first term is controlled by the raw
    corruption comparison in Proposition~\ref{prop:appendix-raw-corruption},
    which is precisely $\mathcal E_T^{\mathrm{nr}}$ from
    \eqref{eq:no-restart-corruption-term}. With probability at least
    \(1-\delta_0\), the second term is controlled by the no-restart clean
    dynamic bound in
    Proposition~\ref{prop:appendix-no-restart-clean}. Combining the pathwise
    corruption bound with the high-probability clean bound proves
    \eqref{eq:raw-total-bound}.
\end{proof}

\begin{proof}[Proof of Theorem~\ref{thm:raw-restart-total}]
    The same decomposition \eqref{eq:appendix-total-split} applies, with the
    conditional laws interpreted as the restarted predictors on the current
    segment. The corruption term is controlled by
    Proposition~\ref{prop:appendix-restart-corruption}, which is precisely
    $\mathcal E_T^{\mathrm{rs}}$ from \eqref{eq:restart-corruption-term}.
    With probability at least \(1-\delta_0\), the clean term is controlled by
    Proposition~\ref{prop:appendix-restart-clean}. Combining the pathwise
    restarted corruption bound with the high-probability clean bound gives
    \eqref{eq:raw-restart-total-bound}.
\end{proof}

\subsection{Proof of the Sublinear-Regime Proposition}
\label{app:sublinear-proof}

\begin{proof}[Proof of Proposition~\ref{prop:sublinear-regime-comparison}]
    We track only polynomial orders and treat fixed constants as irrelevant.
    The localized-comparator radius is assumed to satisfy
    $\varepsilon_T=o(1)$; otherwise the term $L_cT\varepsilon_T$ contributes
    $T\sqrt{\varepsilon_T}$ after the outer square root and is linear when
    $\varepsilon_T$ is fixed.

    First consider the no-restart clean term. Even if the localization term is
    negligible, \eqref{eq:no-restart-clean-term} contains the whole-horizon
    drift contribution
    \[
        \sqrt{
        T\cdot T\sqrt{T A_T^{\mathrm{OT}}}
        }.
    \]
    Under $A_T^{\mathrm{OT}}\asymp T^a$, this scales as
    $T^{(5+a)/4}$, which is not sublinear for $a\ge 0$. Thus the no-restart
    theorem does not yield a sublinear total-regret bound under the present
    analysis. This is an upper-bound statement about the whole-horizon
    comparator proof, not a lower bound ruling out favorable no-restart
    behavior on particular streams.

    For the restarted configuration, first examine corruption. The relevant
    corruption count is segment-local. Assume uniformly over segments that
    $\Lambda_{r,a}\lesssim a^\gamma$ for within-segment age $a$, and use the
    segment-local schedule $\lambda_a\asymp a^{-\beta}$. Then the exponent in
    \eqref{eq:restart-corruption-term} scales as $a^{\gamma-\beta}$. If
    $\beta>\gamma$, then for large $a$,
    \[
        \exp\!\big(O(a^{\gamma-\beta})\big)-1
        =
        O(a^{\gamma-\beta}),
    \]
    and hence
    \[
        \mathcal E_T^{\mathrm{rs}}
        =
        O\!\left(\frac{T}{H}H^{1+\gamma-\beta}\right)
        =
        O\!\left(T H^{\gamma-\beta}\right)
        =
        o(T).
    \]

    It remains to check the restarted clean term. Put $H\asymp T^h$. The
    comparator-complexity contribution has order
    \[
        \sqrt{
        T\cdot \frac{T}{H}\frac{1}{\lambda_{H-1}}
        }
        \asymp
        T^{1-h(1-\beta)/2},
    \]
    which is sublinear whenever $\beta<1$ and $h>0$. The bounded-loss
    learning-rate contribution satisfies
    $\bar\lambda_H\asymp H^{1-\beta}$ for $\beta<1$, and therefore contributes
    \[
        \sqrt{
        T\cdot \frac{T}{H}\bar\lambda_H
        }
        \asymp
        T^{1-h\beta/2},
    \]
    which is sublinear whenever $\beta>0$ and $h>0$.

    The high-probability upgrade adds three confidence contributions to the
    restarted clean bracket. For fixed $\delta_0$, or
    $\log(1/\delta_0)=T^{o(1)}$, these have polynomial orders
    \[
        \frac{\log(1/\delta_0)}{\lambda_{H-1}}
        \asymp
        T^{\beta h+o(1)},
        \qquad
        T\lambda_{H-1}
        \asymp
        T^{1-\beta h},
        \qquad
        \sqrt{T\log(1/\delta_0)}
        =
        T^{1/2+o(1)} .
    \]
    After the outer factor $\sqrt{T(\cdot)}$, these contribute orders
    $T^{(1+\beta h)/2+o(1)}$, $T^{1-\beta h/2}$, and
    $T^{3/4+o(1)}$, respectively. They are all sublinear whenever
    $0<\beta h<1$.

    The restarted segmentwise drift contribution has order
    \[
        \sqrt{
        T\cdot H\sqrt{T A_T^{\mathrm{OT}}}
        }
        \asymp
        T^{(3+a+2h)/4}.
    \]
    This is sublinear exactly when $h<(1-a)/2$. The remaining drift terms have
    orders $T^{(3+a)/4}$ and $T^{(1+a)/2}$, both sublinear when $a<1$.

    \textcolor{black}{The approximation contribution is additive and has order
    \(T^q\), so it is sublinear exactly when \(q<1\).}

    \begingroup
    \color{black}
    Under $\Xi_T(\hat\rho^{\mathrm{rs}})\lesssim T^\zeta$, the two
    cross-stratum contributions have orders
    \[
        D\Xi_T=O(T^\zeta),
        \qquad
        D\sqrt{T\Xi_T}=O\!\left(T^{(1+\zeta)/2}\right).
    \]
    Both are sublinear exactly when $\zeta<1$.

    Under $J_T^{\mathrm{rs}}\lesssim T^b$, the oracle-order switch contribution inside the
    outer square root has order
    \[
        \sqrt{TJ_T^{\mathrm{rs}}}=O\!\left(T^{(1+b)/2}\right),
    \]
    which is sublinear exactly when $b<1$. More concretely,
    $J_T^{\mathrm{rs}}\le D^2(H+1)N_T^{\mathrm{ord}}$. Hence, if
    $H=O(T^h)$ and $N_T^{\mathrm{ord}}=O(T^s)$, then
    $J_T^{\mathrm{rs}}=O(T^{h+s})$, and the transparent sufficient condition
    is $h+s<1$.
    \endgroup

    The sufficient condition
    \[
        \textcolor{black}{0\le a<1},
        \qquad
        \textcolor{black}{q<1},
        \qquad
        \textcolor{black}{\zeta<1},
        \qquad
        \textcolor{black}{b<1},
        \qquad
        \textcolor{black}{0\le\gamma<\beta<1,\qquad 0<h<\frac{1-a}{2}}
    \]
    implies all the requirements above: $\beta>\gamma$ controls corruption,
    \textcolor{black}{$0<\beta<1$ and $0<\beta h<1$ make the clean
    PAC-Bayes and confidence terms sublinear. These follow from the displayed
    regime because $0<h<(1-a)/2\le1/2$; no step requires $\beta<h$.}

    $h<(1-a)/2$ controls the restarted drift term, and
    \textcolor{black}{$\zeta<1$ and $b<1$ control posterior
    cross-stratum leakage and oracle-order switches, respectively.} This proves
    \eqref{eq:restart-sublinear-regime}.

    \begingroup
    \color{black}
    For $u\in\{\mathrm{sw},\mathrm{dc}\}$, the active corruption mass in
    \eqref{eq:localized-memory-control} and
    $\lambda_H\asymp H^{-\beta}$ give
    \[
       \mathcal E_T^u
       =O\!\left(TH^{\gamma-\beta}\right)=o(T)
       \qquad (\gamma<\beta).
    \]
    After the outer Wasserstein conversion, the clean bracket terms in
    \eqref{eq:localized-memory-control} have orders
    \[
    T^{1-h(1-\beta)/2},\quad
    T^{1-h\beta/2},\quad
    T^{(1+\beta h)/2},\quad
    T^{3/4},\quad
    T^{(3+a+2h)/4},\quad
    T^{(3+a)/4}.
    \]
    The remaining transport and approximation terms have orders
    $T^{(1+a)/2}$ and $T^q$. \textcolor{black}{The additional
    cross-stratum terms have orders $T^\zeta$ and
    $T^{(1+\zeta)/2}$, while the oracle-order switch term has order
    $T^{(1+b)/2}$.} Every exponent is below one under
    $\textcolor{black}{0\le a<1}$, $q<1$, \textcolor{black}{$\zeta<1$, $b<1$}, and
    \textcolor{black}{$0\le\gamma<\beta<1$ and $0<h<(1-a)/2$.} Hence sliding-window and discounted history
    are also sublinear whenever their localized bounds satisfy
    \eqref{eq:localized-memory-control}.
    \endgroup
\end{proof}

\begingroup
\color{black}
\paragraph{Scope of the regime result.}
The restart bound is obtained directly from disjoint segments. The
sliding-window and discounted conclusions are conditional on
\eqref{eq:localized-memory-control}, which isolates the overlapping-window and
weighted-potential arguments required by those two updates; it is not an
oracle condition used by the learner. Thus restart is the only rule whose bound
is derived directly here; proving \eqref{eq:localized-memory-control} for the
other two rules remains open.
\endgroup

\begingroup
\color{black}
\raggedbottom
\section{Additional Experimental Figures}
\label{app:additional-experiments}

This section reports supplementary plots for the four memory rules in
Section~\ref{sec:experiments} and the RJMCMC audit.

\begingroup
\color{black}
\subsection{Sensitivity to the Model-Order Cap}
\label{app:model-order-cap}
We hold the generating order at two and vary $p_{\max}\in\{2,3,4\}$,
using the same drifting-corrupted stream construction, five horizons, ten
paired seeds, and inference settings as the principal comparison. Only the
admissible order cap changes. The learner uses the unpenalized clustering
score and fixed order prior. Figure~\ref{fig:model-order-cap} reports
prediction-time off-order mass and cumulative $W_1$ error.

\begin{figure}[H]
\centering
\foreach \cap in {2,3,4}{%
  \foreach \metric/\axislabel/\heading in {
    average/{\Xi_T/T}/{Average leakage},
    cumulative/{\Xi_T}/{Cumulative leakage},
    regret/{R_T}/{Cumulative Wasserstein error}}{%
    \begin{minipage}{0.32\linewidth}
    \centering
    {\scriptsize $p_{\max}=\cap$: \heading}\par
    \begin{tikzpicture}
    \begin{axis}[
      theoryplot,width=\linewidth,height=0.75\linewidth,
      xlabel={Horizon $T$},ylabel={$\axislabel$},
      xmode=log,log basis x=2,ymode=log,log ticks with fixed point,
      xtick={500,2000,8000},xticklabels={500,2000,8000},
      tick label style={font=\tiny},label style={font=\scriptsize},
    ]
    \addplot[blue,mark=*,theory error bars]
      table[x=T,y=full_\metric,y error=full_\metric_err,col sep=comma]
      {\forgettingexperimentdatadir/model_order_cap\cap.csv};
    \addplot[red!75!black,mark=*,theory error bars]
      table[x=T,y=restart_\metric,y error=restart_\metric_err,col sep=comma]
      {\forgettingexperimentdatadir/model_order_cap\cap.csv};
    \addplot[cyan!60!black,mark=triangle*,theory error bars]
      table[x=T,y=sliding_\metric,y error=sliding_\metric_err,col sep=comma]
      {\forgettingexperimentdatadir/model_order_cap\cap.csv};
    \addplot[orange!85!black,mark=square*,theory error bars]
      table[x=T,y=discount_\metric,y error=discount_\metric_err,col sep=comma]
      {\forgettingexperimentdatadir/model_order_cap\cap.csv};
    \end{axis}
    \end{tikzpicture}
    \end{minipage}\hfill
  }\par\medskip
}
{\scriptsize \textcolor{black}{Full history}\quad
\textcolor{red!75!black}{Restart}\quad
\textcolor{cyan!60!black}{Sliding window}\quad
\textcolor{orange!85!black}{Discounted}}
\caption{\textcolor{black}{Sensitivity to the model-order cap with two
generating equal-weight components. Columns show $\Xi_T/T$, $\Xi_T$, and
$R_T=\sum_{t=1}^T W_1(\hat p_t,p_t^*)$. Rows vary the cap. Points are ten-seed
means with 95\% confidence intervals. These are RJMCMC estimates of the
quasi-posterior, with exact one-dimensional predictive $W_1$ evaluation.}}
\label{fig:model-order-cap}
\end{figure}

With larger caps, average off-order mass remains substantial and cumulative
leakage grows approximately linearly over the tested horizons; cumulative
Wasserstein errors also increase relative to the matched cap. Forgetting does
not remove this order-selection limitation. Additional centers can improve
the clustering score without improving the equal-weight predictive law.
Thus these runs do not support unknown-order recovery or the sublinear
leakage condition for excess caps. Enlarging the oracle order instead can
introduce persistent predictive-family approximation error; that alternative
is not measured by this experiment. The finite-time bounds retain both
terms, and matching the cap alone does not establish the sublinear regime.
\endgroup

\subsection{Factorial Baselines and RJMCMC Audit}

\begin{figure}[H]
    \centering
    \foreach \file/\title in {
        factorial_stationary_clean/Stationary clean,
        factorial_stationary_corrupted/Stationary corrupted,
        factorial_drifting_clean/Drifting clean,
        factorial_drifting_corrupted/Drifting corrupted}{%
        \begin{minipage}{0.48\linewidth}
        \centering
        \begin{tikzpicture}
        \begin{axis}[
            theoryplot,
            width=\linewidth,
            height=0.68\linewidth,
            ybar,
            bar width=12pt,
            ylabel={$R_T/T$},
            symbolic x coords={full,restart,sliding,discount},
            xtick=data,
            xticklabels={Full,Restart,Sliding,Discount},
            x tick label style={rotate=20,anchor=east,font=\tiny},
            title={\title},
            enlarge x limits=0.16,
        ]
        \addplot[draw=blue!70!black,fill=blue!65,theory error bars]
            table[x=method,y=normalized,y error minus=normalized_minus,
            y error plus=normalized_plus,col sep=comma]
            {\reviewexperimentdatadir/\file.csv};
        \end{axis}
        \end{tikzpicture}
        \end{minipage}
        \ifnum\pdfstrcmp{\file}{factorial_stationary_corrupted}=0\par\vspace{0.4em}\else\hfill\fi
    }
    \caption{\textcolor{black}{Forgetting baselines in the $2\times2$ factorial
    experiment at $T=4000$. Error bars are 95\% confidence intervals over ten
    paired seeds. Full history is strongest in stationary streams, whereas all
    forgetting rules improve tracking under drift. The sliding window is best
    at this finite horizon; restart remains substantially better than full
    history and is comparable to exponential discounting.}}
    \label{fig:supp-factorial-baselines}
\end{figure}

\paragraph{Implementation details.}

All paper-facing synthetic runs use $T=4000$ unless a horizon sweep is stated,
$p_{\max}=2$, $q(k)\propto e^{-0.1k}$, and centers restricted to
$[-4.75,4.75]$. The default sampler uses $200$ RJMCMC iterations per update,
discards the first quarter as burn-in, and retains $10$ evenly spaced
post-burn-in configurations. We audit $100$, $200$, and $500$ iterations with
$5$ and $10$ retained configurations on the same drifting-corrupted stream.
The audit records normalized regret, wall-clock time per update, overall and
move-specific acceptance rates, and posterior frequencies of $k=1$ and $k=2$.
At $T=4000$, the corruption grid
$\gamma\in\{0.1,0.3,0.5,1.0\}$ contains respectively $2$, $12$, $63$, and
$4000$ corruptions in total, with maximum restart-segment counts $1$, $1$,
$1$, and $60$ for $H=60$; the counts are independent of $\delta$.

For each posterior configuration we use equal-weight uniform kernels of radius
$0.25$, and posterior averaging supplies the configuration weights. In one
dimension, both predictive and true CDFs are piecewise linear with breakpoints
at every kernel endpoint. We therefore compute
$W_1=\int|F_{\widehat p_t}-F_{p_t^*}|$ exactly by integrating on each interval,
including any zero crossing, rather than by drawing evaluation samples. Each
job saves the complete trajectory and diagnostics in a compressed NPZ file;
aggregate CSV files contain paired-seed means and 95\% confidence intervals.
The $220$ jobs are independent CPU jobs, each allocated one CPU and 2~GB memory,
and may run in parallel; no GPU is used.

\begin{figure}[H]
    \centering
    \begin{minipage}{0.48\linewidth}
    \centering
    \begin{tikzpicture}
    \begin{axis}[
        theoryplot,
        width=\linewidth,
        height=0.72\linewidth,
        xlabel={RJMCMC iterations},
        ylabel={$R_T/T$},
        xtick={100,200,500},
        legend pos=north east,
    ]
    \addplot[blue,mark=*,theory error bars]
        table[x=n_mcmc,y=normalized,y error minus=normalized_minus,
        y error plus=normalized_plus,col sep=comma]
        {\reviewexperimentdatadir/rjmcmc_retained_5.csv};
    \addlegendentry{Retained 5}
    \addplot[orange!85!black,mark=square*,theory error bars]
        table[x=n_mcmc,y=normalized,y error minus=normalized_minus,
        y error plus=normalized_plus,col sep=comma]
        {\reviewexperimentdatadir/rjmcmc_retained_10.csv};
    \addlegendentry{Retained 10}
    \end{axis}
    \end{tikzpicture}
    \end{minipage}
    \hfill
    \begin{minipage}{0.48\linewidth}
    \centering
    \begin{tikzpicture}
    \begin{axis}[
        theoryplot,
        width=\linewidth,
        height=0.72\linewidth,
        xlabel={RJMCMC iterations},
        ylabel={Seconds per update},
        ymode=log,
        xtick={100,200,500},
    ]
    \addplot[blue,mark=*,theory error bars]
        table[x=n_mcmc,y=runtime,y error minus=runtime_minus,
        y error plus=runtime_plus,col sep=comma]
        {\reviewexperimentdatadir/rjmcmc_retained_5.csv};
    \addplot[orange!85!black,mark=square*,theory error bars]
        table[x=n_mcmc,y=runtime,y error minus=runtime_minus,
        y error plus=runtime_plus,col sep=comma]
        {\reviewexperimentdatadir/rjmcmc_retained_10.csv};
    \end{axis}
    \end{tikzpicture}
    \end{minipage}
    \caption{\textcolor{black}{RJMCMC audit on the drifting-corrupted stream.
    Left: normalized regret. Right: wall-clock time per update. The default
    $200$-iteration, $10$-draw setting gives $R_T/T=0.257$ at $0.027$ seconds
    per update. Its within-model, birth, death, and overall acceptance rates
    are $0.495$, $0.413$, $0.007$, and $0.255$; the posterior assigns $0.984$
    mass to the true order $k=2$.}}
    \label{fig:supp-rjmcmc-audit}
\end{figure}

\textcolor{black}{These measurements are computational diagnostics rather than
a formal convergence guarantee. In particular, death-move acceptance is low,
although regret and posterior model-order frequencies remain stable across the
tested sampling budgets.}

\subsection{Robustness and Parameter Sensitivity}

The following experiments use the same exact one-dimensional Wasserstein
evaluation as Figure~\ref{fig:main-regret-scaling}. We report normalized
regret $R_T/T$ only as a finite-horizon comparison metric;
\textcolor{black}{cumulative $R_T$ in the main figure illustrates
finite-range scaling, not the asymptotic guarantee.}

\begin{figure}[H]
    \centering
    \foreach \method/\title in {full/Full history,restart/Restart,sliding/Sliding window,discount/Discounted}{%
    \begin{minipage}[t]{0.48\linewidth}
        \centering
        \begin{tikzpicture}
        \begin{axis}[
            theoryplot,
            title={\title},
            width=0.88\linewidth,
            height=0.62\linewidth,
            view={0}{90},
            colormap/viridis,
            point meta min=0.25,
            point meta max=3.20,
            xlabel={Frequency exponent $\gamma$},
            ylabel={Corruption size $\delta$},
            xtick={0,1,2,3},
            xticklabels={0.1,0.3,0.5,1.0},
            ytick={0,1,2},
            yticklabels={0.5,2,4},
        ]
        \addplot[matrix plot*,mesh/cols=4,point meta=explicit]
            table[x=gamma_index,y=delta_index,meta=normalized,col sep=comma]
            {\forgettingexperimentdatadir/corruption_\method.csv};
        \end{axis}
        \end{tikzpicture}
    \end{minipage}
    \ifnum\pdfstrcmp{\method}{restart}=0\par\vspace{0.2em}\else\hfill\fi}
    \caption{\textcolor{black}{Effect of corruption at $T=4000$. The corruption count scales as
    $T^\gamma$ and each corrupted observation has displacement bounded by
    $\delta$. All panels use the same color scale for $R_T/T$.}}
    \label{fig:supp-corruption-grid}
\end{figure}

\begin{figure}[H]
    \centering
    \foreach \bval in {0.1,0.2,0.3}{%
    \begin{minipage}{0.32\linewidth}\centering
    \begin{tikzpicture}\begin{axis}[
        theoryplot,width=\linewidth,height=0.78\linewidth,
        title={$\beta=\bval$},xlabel={$c$ in $H=cT^{0.3}$},
        ylabel={$R_T/T$},xtick={1,3,5,10},
        legend style={font=\tiny},legend pos=north east]
    \addplot[red!75!black,mark=*,theory error bars]
        table[x=c,y=restart,y error minus=restart_minus,y error plus=restart_plus,col sep=comma]
        {\forgettingexperimentdatadir/memory_beta_\bval.csv};
    \addlegendentry{Restart}
    \addplot[cyan!60!black,mark=triangle*,theory error bars]
        table[x=c,y=sliding,y error minus=sliding_minus,y error plus=sliding_plus,col sep=comma]
        {\forgettingexperimentdatadir/memory_beta_\bval.csv};
    \addlegendentry{Sliding}
    \addplot[orange!85!black,mark=square*,theory error bars]
        table[x=c,y=discount,y error minus=discount_minus,y error plus=discount_plus,col sep=comma]
        {\forgettingexperimentdatadir/memory_beta_\bval.csv};
    \addlegendentry{Discounted}
    \end{axis}\end{tikzpicture}\end{minipage}\hfill}
    \caption{\textcolor{black}{Memory-scale sensitivity at $T=4000$ for restart,
    sliding-window, and discounted scores. Error bars are 95\% confidence
    intervals over ten seeds.}}
    \label{fig:supp-restart-window}
\end{figure}

\begin{figure}[H]
    \centering
    \begin{tikzpicture}
    \begin{axis}[
        theoryplot,
        width=0.72\linewidth,
        height=0.42\linewidth,
        ybar,
        bar width=6pt,
        ylabel={$R_T/T$},
        symbolic x coords={constant,sqrt_log_over_t,power},
        xtick=data,
        xticklabels={Constant,$\sqrt{\log(t)/t}$,$t^{-0.2}$},
        enlarge x limits=0.25,
        legend columns=4,
        legend pos=north west,
    ]
    \addplot[fill=blue!75,draw=blue,theory error bars]
        table[x=schedule,y=full,y error minus=full_minus,
        y error plus=full_plus,col sep=comma]
        {\forgettingexperimentdatadir/lambda_sensitivity.csv};
    \addlegendentry{Full history}
    \addplot[fill=red!70!black,draw=red!80!black,theory error bars]
        table[x=schedule,y=restart,y error minus=restart_minus,
        y error plus=restart_plus,col sep=comma]
        {\forgettingexperimentdatadir/lambda_sensitivity.csv};
    \addlegendentry{Restart}
    \addplot[fill=cyan!60!black,draw=cyan!50!black,theory error bars]
        table[x=schedule,y=sliding,y error minus=sliding_minus,
        y error plus=sliding_plus,col sep=comma]
        {\forgettingexperimentdatadir/lambda_sensitivity.csv};
    \addlegendentry{Sliding}
    \addplot[fill=orange!85!black,draw=orange!70!black,theory error bars]
        table[x=schedule,y=discount,y error minus=discount_minus,
        y error plus=discount_plus,col sep=comma]
        {\forgettingexperimentdatadir/lambda_sensitivity.csv};
    \addlegendentry{Discounted}
    \end{axis}
    \end{tikzpicture}
    \caption{\textcolor{black}{Effect of the learning-rate schedule at $T=4000$, with
    $a=0.2$, $\gamma=0.1$, $\delta=1$, and memory scale $H=5T^{0.3}$.
    Error bars are 95\% confidence intervals over ten seeds.}}
    \label{fig:supp-learning-rate}
\end{figure}

\subsection{Predictive-Family Approximation}
\label{app:approximation-experiment}

\textcolor{black}{We isolate the approximation term from drift and corruption.
\textcolor{black}{For each horizon $T$, a separate uncorrupted run has
a fixed true law with} two fixed centers
$c_1=-1.5$ and $c_2=1.5$ with uniform kernels of radius $0.25$:
\[
p_t^*=\left(\frac12+\eta_T\right)K(\cdot,c_1)
      +\left(\frac12-\eta_T\right)K(\cdot,c_2),
\qquad
\eta_T=0.3\left(\frac{T}{500}\right)^{q-1}.
\]
The learner retains the equal-weight predictive family used throughout the
paper. Since the supports are disjoint and separated by $3$, the resulting
gap is exact: $\alpha_t=3\eta_T$ and
$A_T^{\mathrm{app}}=3T\eta_T\asymp T^q$. Thus
$A_T^{\mathrm{OT}}=0$ and there is no corruption; only model mismatch varies.
We test $q\in\{0.50,0.75,0.90,1.00\}$ and a realizable control over
$T\in\{500,1000,2000,4000,8000\}$. All curves are paired-seed means over ten
runs and use exact one-dimensional $W_1$.}

\begin{figure}[H]
    \centering
    \begin{minipage}{0.48\linewidth}
    \centering
    \textbf{\small (a) Cumulative Wasserstein regret}
    \begin{tikzpicture}
    \begin{axis}[
        theoryplot,width=\linewidth,height=0.72\linewidth,
        xlabel={Horizon $T$},ylabel={$R_T$},
        xmode=log,log basis x=2,ymode=log,
        xtick={500,1000,2000,4000,8000},
        xticklabels={500,1000,2000,4000,8000},
        legend columns=2,legend pos=north west]
    \addplot[black,mark=*] table[x=T,y=real_regret,col sep=comma]
        {\forgettingexperimentdatadir/approximation_scaling.csv};
    \addlegendentry{Realizable}
    \addplot[blue,mark=square*] table[x=T,y=q05_regret,col sep=comma]
        {\forgettingexperimentdatadir/approximation_scaling.csv};
    \addlegendentry{$q=0.50$}
    \addplot[cyan!60!black,mark=triangle*] table[x=T,y=q075_regret,col sep=comma]
        {\forgettingexperimentdatadir/approximation_scaling.csv};
    \addlegendentry{$q=0.75$}
    \addplot[orange!85!black,mark=diamond*] table[x=T,y=q09_regret,col sep=comma]
        {\forgettingexperimentdatadir/approximation_scaling.csv};
    \addlegendentry{$q=0.90$}
    \addplot[red!75!black,mark=pentagon*] table[x=T,y=q1_regret,col sep=comma]
        {\forgettingexperimentdatadir/approximation_scaling.csv};
    \addlegendentry{$q=1.00$}
    \end{axis}
    \end{tikzpicture}
    \end{minipage}
    \hfill
    \begin{minipage}{0.48\linewidth}
    \centering
    \textbf{\small (b) Average Wasserstein regret}
    \begin{tikzpicture}
    \begin{axis}[
        theoryplot,width=\linewidth,height=0.72\linewidth,
        xlabel={Horizon $T$},ylabel={$R_T/T$},
        xmode=log,log basis x=2,
        xtick={500,1000,2000,4000,8000},
        xticklabels={500,1000,2000,4000,8000}]
    \addplot[black,mark=*] table[x=T,y=real_norm,col sep=comma]
        {\forgettingexperimentdatadir/approximation_scaling.csv};
    \addplot[blue,mark=square*] table[x=T,y=q05_norm,col sep=comma]
        {\forgettingexperimentdatadir/approximation_scaling.csv};
    \addplot[cyan!60!black,mark=triangle*] table[x=T,y=q075_norm,col sep=comma]
        {\forgettingexperimentdatadir/approximation_scaling.csv};
    \addplot[orange!85!black,mark=diamond*] table[x=T,y=q09_norm,col sep=comma]
        {\forgettingexperimentdatadir/approximation_scaling.csv};
    \addplot[red!75!black,mark=pentagon*] table[x=T,y=q1_norm,col sep=comma]
        {\forgettingexperimentdatadir/approximation_scaling.csv};
    \end{axis}
    \end{tikzpicture}
    \end{minipage}
    \caption{\textcolor{black}{Controlled predictive-family mismatch. For every
    $q<1$, normalized regret decreases with the horizon; at $q=1$ it remains
    approximately constant.}}
    \label{fig:supp-approximation-scaling}
\end{figure}

\begin{figure}[H]
    \centering
    \begin{tikzpicture}
    \begin{axis}[
        theoryplot,width=0.72\linewidth,height=0.40\linewidth,
        ybar,bar width=18pt,
        xlabel={Predictive-family mismatch},
        ylabel={Fitted exponent $\widehat\alpha$ in $R_T\propto T^{\widehat\alpha}$},
        xmin=-0.6,xmax=4.6,ymin=0.5,ymax=1.06,
        xtick={0,1,2,3,4},
        xticklabels={Realizable,$q=0.50$,$q=0.75$,$q=0.90$,$q=1.00$},
        x tick label style={rotate=18,anchor=east,font=\scriptsize}]
    \addplot[draw=blue!70!black,fill=blue!60]
        table[x=x,y=alpha,col sep=comma]
        {\forgettingexperimentdatadir/approximation_exponents.csv};
    \addplot[red!75!black,dashed,thick,sharp plot]
        coordinates {(-0.55,1) (4.55,1)};
    \node[anchor=south east,font=\scriptsize,text=red!75!black]
        at (axis cs:4.55,1) {linear boundary};
    \end{axis}
    \end{tikzpicture}
    \caption{\textcolor{black}{Regret growth under model misspecification.
    \textcolor{black}{Fitted exponents below one indicate slower-than-linear
    growth over the tested horizons; the $q=1$ fit lies near the linear boundary.
    These fits illustrate, rather than test, the asymptotic regime; the upper
    bound is not numerically informative here.}}}
    \label{fig:supp-approximation-regime}
\end{figure}

\endgroup

\subsection{Real-Data SPY Return Streams}
\label{app:real-data-experiments}

\begingroup
\setlength{\abovedisplayskip}{0.45em}
\setlength{\belowdisplayskip}{0.45em}
\setlength{\abovedisplayshortskip}{0.25em}
\setlength{\belowdisplayshortskip}{0.35em}

We also evaluate the raw and restarted predictors in online prediction runs on
real daily SPY market data. The learner processes the stream sequentially and
uses only past observations for its posterior update and prediction. The
distributional target used for plotting is constructed only after the fact, as
an offline evaluation proxy for the unknown local data law.
For the one-dimensional stream, the input sample is the close-price log return
\[
    x_t=\log(\mathrm{Close}_t)-\log(\mathrm{Close}_{t-1})\in\RR.
\]
For a five-dimensional financial stream one may, for example, use a multi-asset
return vector
\[
    \begin{aligned}
    x_t &=
    (r_t^{\mathrm{SPY}},r_t^{\mathrm{QQQ}},r_t^{\mathrm{IWM}},
    r_t^{\mathrm{TLT}},r_t^{\mathrm{GLD}})
    \in\RR^5,\\
    r_t^a&=\log(P_t^a)-\log(P_{t-1}^a).
    \end{aligned}
\]
In the real-data experiment reported below, the available five-dimensional
stream is instead the SPY OHLCV return vector
\[
    x_t =
    (r_t^{\mathrm{Open}},r_t^{\mathrm{High}},r_t^{\mathrm{Low}},
    r_t^{\mathrm{Close}},r_t^{\mathrm{Volume}})
    \in\RR^5,
\]
with each component defined as the corresponding log difference and then
standardized.

The true time-varying law $p_t^*$ is not observed in real data. We therefore
compare the predictive distribution against a rolling empirical future-window
proxy. This proxy is not available to the online learner and is used only for
post-hoc evaluation. For window length $w=20$, this proxy is
\[
    \widetilde p_t=\frac{1}{20}\sum_{j=0}^{19}\delta_{x_{t+j}}.
\]
The plotted distributional quantity compares $\hat p_t$ to $\widetilde p_t$:
empirical $W_1$ in one dimension and sliced $W_1$ in the five-dimensional
stream. Both real-data runs use $\lambda_t=0.1\sqrt{\log(t)/t}$, restart
interval $H=25$, $100$ RJMCMC iterations per update, a warm-up of $40$
observations, $200$ evaluation samples, $200$ evaluation burn-in steps, and
$100$ random projections for the five-dimensional sliced-$W_1$ proxy. We also
report mean-prediction error by sampling
$\hat X_t^{(1)},\ldots,\hat X_t^{(m)}\sim\hat p_t$, forming
\[
    \hat\mu_t=\frac{1}{m}\sum_{i=1}^m\hat X_t^{(i)},\qquad
    e_t=\|\hat\mu_t-x_t\|_2,
\]
and plotting the cumulative, average, and log-cumulative versions of this
error.

\endgroup

\begin{figure}[H]
    \centering
    \textbf{\small Close-return distributional proxy}

    \begin{minipage}{0.315\linewidth}
        \centering
        \plotrealexperimentcsv{real_spy_close_average_distance_proxy.csv}{t}
        {$t$}{$R_t/t$}{Average proxy}
    \end{minipage}
    \hfill
    \begin{minipage}{0.315\linewidth}
        \centering
        \plotrealexperimentcsv{real_spy_close_cumulative_distance_proxy.csv}{t}
        {$t$}{$R_t$}{Cumulative proxy}
    \end{minipage}
    \hfill
    \begin{minipage}{0.315\linewidth}
        \centering
        \plotrealexperimentcsv{real_spy_close_log1p_cumulative_distance_proxy.csv}{t}
        {$t$}{$\log(1+R_t)$}{Log cumulative proxy}
    \end{minipage}

    \vspace{0.25em}
    \textbf{\small Close-return mean-prediction error}

    \begin{minipage}{0.315\linewidth}
        \centering
        \plotrealexperimentcsv{real_spy_close_mean_average_error.csv}{t}
        {$t$}{Average error}{Average mean error}
    \end{minipage}
    \hfill
    \begin{minipage}{0.315\linewidth}
        \centering
        \plotrealexperimentcsv{real_spy_close_mean_cumulative_error.csv}{t}
        {$t$}{Cumulative error}{Cumulative mean error}
    \end{minipage}
    \hfill
    \begin{minipage}{0.315\linewidth}
        \centering
        \plotrealexperimentcsv{real_spy_close_mean_log1p_cumulative_error.csv}{t}
        {$t$}{$\log(1+\mathrm{error})$}{Log mean error}
    \end{minipage}

    \vspace{0.25em}
    \textbf{\small OHLCV-return distributional proxy}

    \begin{minipage}{0.315\linewidth}
        \centering
        \plotrealexperimentcsv{real_spy_ohlcv_average_distance_proxy.csv}{t}
        {$t$}{$R_t/t$}{Average proxy}
    \end{minipage}
    \hfill
    \begin{minipage}{0.315\linewidth}
        \centering
        \plotrealexperimentcsv{real_spy_ohlcv_cumulative_distance_proxy.csv}{t}
        {$t$}{$R_t$}{Cumulative proxy}
    \end{minipage}
    \hfill
    \begin{minipage}{0.315\linewidth}
        \centering
        \plotrealexperimentcsv{real_spy_ohlcv_log1p_cumulative_distance_proxy.csv}{t}
        {$t$}{$\log(1+R_t)$}{Log cumulative proxy}
    \end{minipage}

    \vspace{0.25em}
    \textbf{\small OHLCV-return mean-prediction error}

    \begin{minipage}{0.315\linewidth}
        \centering
        \plotrealexperimentcsv{real_spy_ohlcv_mean_average_error.csv}{t}
        {$t$}{Average error}{Average mean error}
    \end{minipage}
    \hfill
    \begin{minipage}{0.315\linewidth}
        \centering
        \plotrealexperimentcsv{real_spy_ohlcv_mean_cumulative_error.csv}{t}
        {$t$}{Cumulative error}{Cumulative mean error}
    \end{minipage}
    \hfill
    \begin{minipage}{0.315\linewidth}
        \centering
        \plotrealexperimentcsv{real_spy_ohlcv_mean_log1p_cumulative_error.csv}{t}
        {$t$}{$\log(1+\mathrm{error})$}{Log mean error}
    \end{minipage}

    \caption{Real-data SPY experiments. The close-return panels use the
    one-dimensional log-return stream. The OHLCV panels use the five-dimensional
    standardized SPY Open, High, Low, Close, and Volume log-return stream.
    Distributional panels compare the online predictive law with the offline
    rolling empirical future-window proxy; mean-error panels compare the
    predictive mean with the realized standardized return vector.}
    \label{fig:real-spy-all}
\end{figure}

\end{document}